\documentclass{article}
\usepackage[utf8]{inputenc}
\usepackage[T1]{fontenc}
\usepackage{amsmath, amssymb, amsthm}
\newtheorem{remark}{Remark}
\usepackage{thmtools}
\usepackage{algorithm}
\usepackage{algpseudocode}
\usepackage{mathtools}
\usepackage{hyperref}
\usepackage{xcolor}
\definecolor{firebrick}{RGB}{178,34,34}
\usepackage{adjustbox}
\usepackage{newunicodechar}
\newunicodechar{─}{-}
\usepackage{newunicodechar}
\newunicodechar{═}{=}
\usepackage{booktabs}
\usepackage{enumitem}
\usepackage{tikz}
\usetikzlibrary{arrows.meta, positioning, shapes.geometric, calc}
\usepackage{mathpazo}
\usepackage{booktabs}
\usepackage{multirow}
\usepackage{makecell}
\usepackage{xcolor}
\usepackage{colortbl}
\usepackage{array}

\usepackage{thm-restate}
\usepackage[most]{tcolorbox}
\tcbset{
    boxsep=0mm,
    boxrule=0pt,
    colframe=white,
    arc=0mm,
    left=0.5mm,
    right=0.5mm
}

\definecolor{neural}{RGB}{220, 80, 60}
\definecolor{symbolic}{RGB}{50, 120, 200}
\definecolor{compose}{RGB}{40, 160, 80}
\definecolor{lightgray}{RGB}{245, 245, 245}
\definecolor{weak}{RGB}{255, 237, 215}
\definecolor{mid}{RGB}{215, 232, 255}
\definecolor{strong}{RGB}{215, 250, 222}
\definecolor{best}{RGB}{180, 230, 180}
\definecolor{rlmbest}{RGB}{200, 215, 255}
 
\newcommand{\lwin}[1]{\cellcolor{best}\textbf{#1}}
\newcommand{\rwin}[1]{\cellcolor{rlmbest}\textbf{#1}}
\newcommand{\up}[1]{\textcolor{green!50!black}{\textbf{#1}}}

\newtheorem{theorem}{Theorem}

\newtheorem{corollary}[theorem]{Corollary}
\newtheorem{definition}{Definition}

\newcommand{\lRLM}{\lambda\text{-}\textsf{RLM}}
\newcommand{\lRLMs}{\lambda\text{-}\textsf{RLMs}}
\newcommand{\LLM}{\mathcal{M}}
\newcommand{\REPL}{\mathcal{E}}
\newcommand{\cost}{\mathcal{C}}
\newcommand{\accuracy}{\mathcal{A}}

\newcommand{\vocab}{\Sigma}

\newcommand{\Map}{\textsc{Map}}
\newcommand{\Filter}{\textsc{Filter}}
\newcommand{\Reduce}{\textsc{Reduce}}
\newcommand{\Split}{\textsc{Split}}
\newcommand{\Peek}{\textsc{Peek}}
\newcommand{\Concat}{\textsc{Concat}}
\newcommand{\Cross}{\textsc{Cross}}

\definecolor{weak}{RGB}{255, 235, 210}
\definecolor{mid}{RGB}{210, 230, 255}
\definecolor{strong}{RGB}{210, 255, 220}

\algrenewcommand\algorithmicrequire{\textbf{Input:}}
\algrenewcommand\algorithmicensure{\textbf{Output:}}

\title{\textbf{$\lambda$-RLM: Lambda-Calculus Recursive Language Models} \\[0.5em]
\large End-to-End Algorithm Specification}
\author{}
\date{}

\usepackage[table]{xcolor}
\usepackage[main, final]{neurips_2025}
\newcommand{\pykw}[1]{\textcolor{compose}{\texttt{#1}}}
\newcommand{\pyfn}[1]{\textcolor{blue}{\texttt{#1}}}

\usepackage[utf8]{inputenc} %
\usepackage[T1]{fontenc}    %
\usepackage{hyperref}       %
\usepackage{url}            %
\usepackage{booktabs}       %
\usepackage{amsfonts}       %
\usepackage{nicefrac}       %
\usepackage{microtype}      %
\usepackage{xcolor}

\usepackage{algorithm}
\usepackage{algpseudocode}

\usepackage{graphicx} %
\usepackage{amsfonts} %
\usepackage{amssymb} %
\usepackage{amsthm}
\usepackage{amsmath} %
\newtheorem{assumption}{Assumption}
\usepackage{hyperref}

\title{The $\mathbf{Y}$-Combinator for LLMs: \\ Solving Long-Context Rot with $\lambda$-Calculus}

\author{%
\hspace{1mm}Amartya Roy\\
	The School of Interdisciplinary Research \\
	Indian Institute of Technology (IIT) Delhi\\
    Robert Bosch GmbH, India\\
	\texttt{srz248670@iitd.ac.in} \\
  \And
  Rasul Tutunov \\
  Huawei Noah's Ark Lab\\
  \texttt{rasul.tutunov@huawei.com}
  \And
  Xiaotong Ji \\ 
  Huawei Noah's Ark Lab \\ 
  \texttt{xiaotong.ji1@h-partners.com} \\
  \And 
  Matthieu Zimmer \\
  Huawei Noah's Ark Lab \\
  \texttt{matthieu.zimmer@huawei.com} 
  \And 
  Haitham Bou-Ammar \\
  Huawei Noah's Ark \\ 
  UCL Centre for Artificial Intelligence\\
  \texttt{haitham.ammar@huawei.com}
}

\begin{document}

\maketitle

\begin{abstract}
LLMs are increasingly used as general-purpose reasoners, but long inputs remain bottlenecked by a fixed context window. Recursive Language Models (RLMs) address this by externalising the prompt and recursively solving subproblems. Yet existing RLMs depend on an open-ended read–eval–print loop (REPL) in which the model generates arbitrary control code, making execution difficult to verify, predict, and analyse.

We introduce $\lRLM$, a framework for long-context reasoning that replaces free-form recursive code generation with a typed functional runtime grounded in $\lambda$-calculus. It executes a compact library of pre-verified combinators and uses neural inference only on bounded leaf subproblems, turning recursive reasoning into a structured functional program with explicit control flow.
We show that $\lRLM$ admits formal guarantees absent from standard RLMs, including termination, closed-form cost bounds, controlled accuracy scaling with recursion depth, and an optimal partition rule under a simple cost model. Empirically, across four long-context reasoning tasks and nine base models, $\lRLM$ outperforms standard RLM in 29 of 36 model-task comparisons, improves average accuracy by up to $\textbf{+21.9}$ points across model tiers, and reduces latency by up to $\textbf{4.1}\times$. These results show that typed symbolic control yields a more reliable and efficient foundation for long-context reasoning than open-ended recursive code generation. The complete implementation of $\lRLMs$, is open-sourced for the community at: \href{https://github.com/lambda-calculus-LLM/lambda-RLM}{\texttt{github.com/lambda-calculus-LLM/lambda-RLM}}.
\end{abstract}

\section{Introduction}
Large language models (LLMs) are increasingly used as general-purpose problem solvers \citep{brown2020languagemodelsfewshotlearners, yao2023treethoughtsdeliberateproblem,mower2024rosllmrosframeworkembodied,zimmer2025bourbakiselfgeneratedgoalconditionedmdps, ji2026scalablepowersamplingunlocking}, yet one of their most fundamental bottlenecks remains unchanged: \emph{a Transformer consumes a fixed-length context window} \citep{dai2019transformerxlattentivelanguagemodels}. When inputs exceed this limit, e.g., long documents, codebases, multi-file repositories, or large collections of evidence, naïvely truncating context or relying on sliding-window prompting forces the model to “forget” early information and often breaks tasks that require global consistency or systematic evidence gathering \citep{liu2023lostmiddlelanguagemodels, wang2024limitssurveytechniquesextend}. In response, a growing line of work reframes long-context reasoning as inference-time scaling: rather than increasing model parameters or training new architectures, we can scale computation at inference by decomposing problems into smaller subproblems and composing their solutions \citep{zhou2023leasttomostpromptingenablescomplex, yao2023reactsynergizingreasoningacting, yao2023treethoughtsdeliberateproblem, yang2025testtimescalingllmssurvey}.

\begin{figure}[t!]
    \centering
    \includegraphics[trim={3em 7em 3em 8em}, clip, width=\linewidth]{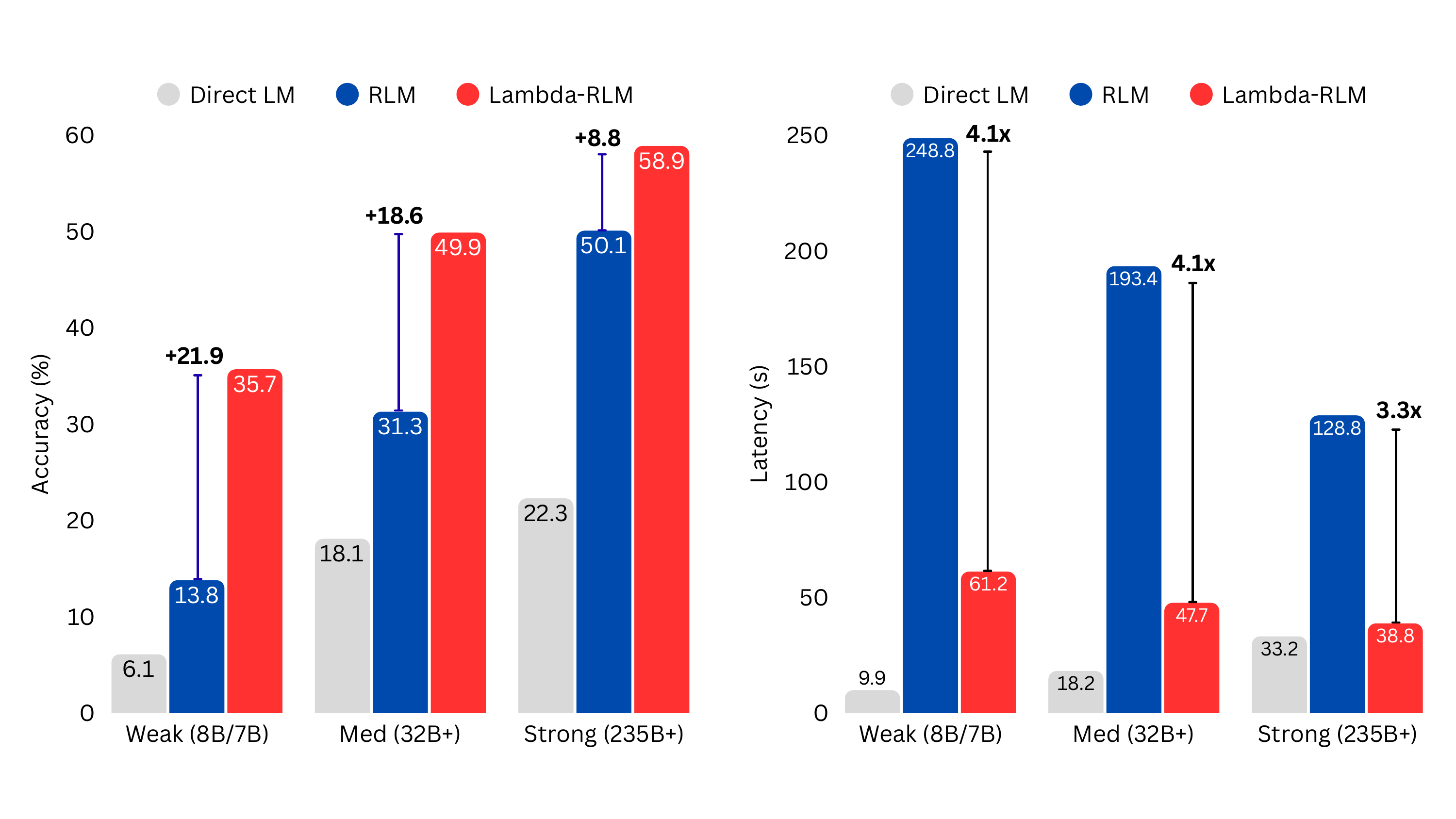}
    \caption{A summary of our results comparing $\lRLMs$ to base LLMs and recursive LLMs. Those results demonstrate improvements reaching $\textbf{+21.9}$ in accuracy, with $\textbf{4.1}\times $ in latency reductions.}
    \label{fig:placeholder}
\end{figure}
A particularly compelling recent proposal is Recursive Language Models (RLMs), which argues that arbitrarily long user prompts should not be fed into the neural network directly \citep{zhang2026recursivelanguagemodels}. Instead, the prompt should be treated as part of an external environment that the model can interact with symbolically. Concretely, RLM initialises a programming environment (a REPL) in which the prompt is stored as a variable; the LLM then writes code to peek into the prompt, decompose it into slices, and recursively invoke itself on those slices as needed. This simple interface, prompt-as-environment plus symbolic recursion, enables models to handle inputs far beyond their native context length while retaining a standard “string-in, string-out” API.

However, RLM’s power comes with a practical cost: it relies on an LLM-driven control loop that emits and executes arbitrary code until the model decides it has finished. This open-ended REPL loop is difficult to bound and audit. In practice, it creates several failure modes that are orthogonal to the underlying reasoning task: code may not parse or may crash at runtime; recursion may be invoked excessively; intermediate outputs may be malformed; and computation may become unpredictable due to the model’s own control-flow decisions. More broadly, giving an LLM unrestricted freedom to program its own execution introduces an undesirable coupling between what the model knows and how it is allowed to search and compose evidence.

In this work, we propose $\lRLM$, a framework that retains the key insight of RLM, prompt-as-environment with recursive decomposition but replaces open-ended code generation with a typed, functional runtime grounded in $\lambda$-Calculus. $\lRLM$ expresses all control flow through a small library of deterministic, compositional operators (e.g., $\Split$, $\Map$, $\Filter$, $\Reduce$) that are pre-verified and loaded into the REPL before execution. The base language model $\LLM$ is invoked only at the leaves of the recursion, on sub-prompts that are guaranteed to fit within its context window $K$; all higher-level decisions i) how to split, ii) how many chunks, iii) when to stop, iv)how to compose are made by a planner and executed symbolically, without any LLM-generated code. Recursion is encoded as a fixed-point over this operator library (Section~\ref{sec:framework}), and the planner enforces predictable execution: maximum depth $d = \lceil \log_{k^*}(n/\tau^*) \rceil$, a pre-computed number of $\LLM$ calls, and deterministic composition at every level. As a result, $\lRLM$ separates \emph{semantic reasoning} from \emph{structural control}: the model contributes understanding only where it is needed; at leaf sub-problems small enough to process reliably while all orchestration is handled by an auditable, deterministic controller with formal guarantees on termination, cost, and accuracy.

We choose $\lambda$-Calculus as our foundation because it provides a minimalist yet universal interface for hierarchical reasoning that other formalisms lack. While Finite State Machines (FSMs) are insufficient for the arbitrary recursion depths required in complex document decomposition, and Planning Domain Definition Languages (PDDL) are optimised for state-space search rather than data transformation, $\lambda$-Calculus treats the prompt as a first-class functional object. Crucially, by utilising fixed-point combinators (e.g., the $Y$-combinator), $\lambda$-RLM "ties the knot" of recursion without requiring the LLM to manage function names or global state, effectively eliminating the reference errors and non-termination failures common in open-ended REPL loops.

Our design choices yield three benefits. First, $\lRLM$ provides termination by construction under mild conditions on the splitting operator, eliminating a common class of non-termination and runaway-execution failures in agentic scaffolds. Second, it yields predictable computation: we can bound the number of oracle calls and the total work as a function of input size and the chosen decomposition policy. Third, it improves reliability by reducing the number of “critical decisions” delegated to the language model. We evaluate $\lRLM$ on long-context task settings, using RLM as a primary baseline. In short, our contributions can be summarised as: \textit{i)} We introduce $\lRLM$, a typed functional runtime for prompt-as-environment long-context reasoning with recursion expressed as a fixed-point over deterministic combinators; \textit{ii)} We formalise an operational semantics and prove termination and cost bounds under standard size-decreasing decomposition assumptions; and \textit{iii)} We empirically compare against RLM, demonstrating improved reliability and more predictable compute while improving task performance.

We validate $\lRLM$ on four long-context task families spanning search, aggregation, pairwise reasoning, and code understanding, across nine base models and context lengths up to 128K. Compared with normal RLM, $\lRLM$ wins in 29/36 model-task comparisons ($\textbf{81}\%$ overall), improves average accuracy by up to $\textbf{+21.9}$ points on weak models and $\textbf{+18.6}$ points on medium models, and delivers consistent latency reductions of $\textbf{3.3}\times$ to $\textbf{4.1}\times$. On the most structurally demanding benchmark, OOL-Pairs, the gain reaches $\textbf{+28.6}$ points with a $\textbf{6.2}\times$ speedup. These results show that constraining control flow to a typed combinator runtime not only improves predictability but also leads to substantial empirical gains over open-ended recursive code generation.
\section{A Short Primer on $\lambda$-Calculus}\label{Sec:Primal}
The lambda calculus is a minimal formal language for describing computation using only \emph{functions and functional operations}. We include a brief primer here because $\lRLM$ uses a functional view of control flow: recursion and composition are expressed as combinations of small operators, rather than as an LLM-driven loop that generates arbitrary code.

 We use $\texttt{Exp}$ to denote the set of (untyped) lambda-calculus expressions, and $\texttt{Var}$ to denote a countable set of variable names. The grammar is defined by: 
\begin{equation*}
    \texttt{Exp} ::= x \ | \ (\lambda x . \  \texttt{Exp}) \ | \ (\texttt{Exp} \ \texttt{Exp}), \ \ \ x \in \texttt{Var}. 
\end{equation*}
Intuitively, the above is saying that every expression is one of three forms: \textit{i)} A variable which is a placeholder: it gains meaning when it is bound by a function or substituted during evaluation; \textit{ii)} An abstraction (function definition): If $e \in \texttt{Exp}$ and $x \in \texttt{Var}$, then $\lambda x. \ e \in \texttt{Exp}$. This is to be read as: `` a function that takes an argument $x \in \texttt{Var}$ and returns $e \in \texttt{Exp}$. For instance, we could define an identity function as: $\texttt{id} = \lambda x. \ x$, or a constant function that returns its first argument as: $\texttt{const} = \lambda x. \ \lambda y. \ x.$; and \textit{iii)} An Application (functional call): If $e_1, e_2 \in \texttt{Exp}$, then $(e_1, e_2) \in \texttt{Exp}$, which is to be read as ``apply $e_1$ to $e_2$''. For example, the abstraction $(\lambda x. x) \ y$ applies the identity function to $y$. By convention, application associates to the left: $ e_1 \ e_2 \ e_3 \equiv \ (e_1 \ e_2) \ e_3$. In other words, $f \ a \ b$ means ``first apply $f$ to $a$, then apply the result to $b$''. We may omit the outer parentheses when unambiguous.  

Syntax tells us what expressions look like. Evaluation tells us how expressions compute. In the untyped lambda calculus, the central computational rule is $\beta$-reduction, which formalises what it means to apply a function to an argument. If we have a function $\lambda x. \ e$ and we apply it to an argument $a$, we reduce by substituting the argument $a$ for the variable $x$ inside the body $e$: 
\begin{equation*}
    (\lambda x.\ e)\ a \xrightarrow{\beta} e[x := a], \ \ \text{where $e[x := a]$ takes $e$ and replaces every free occurrence of $x$ by $a$. }
\end{equation*}
In other words, $\beta$-reduction is just a function application as substitution, exactly like evaluating a function call. Let us develop some examples to understand $\beta$-reduction better: 

\begin{tcolorbox}[colback=lightgray, colframe=blue!50, title=\textbf{Examples of $\beta$-Reduction}]
\small
\textbf{Identity.} Let $\texttt{id} = \lambda x.\, x$. Then $(\lambda x.\, x)\, y \xrightarrow{\beta} y$: applying the identity returns its input.

\textbf{Constant function.} Let $\texttt{const} = \lambda x.\, \lambda y.\, x$. By left associativity, $\texttt{const}\; a\; b = ((\lambda x.\, \lambda y.\, x)\; a)\; b$. Reducing step by step:
\[
    \underbrace{(\lambda x.\, \lambda y.\, x)\; a}_{\xrightarrow{\beta}\; \lambda y.\, a}
    \; b
    \;\xrightarrow{\beta}\; a.
\]
The outer $\lambda$ binds $a$, yielding a constant function $\lambda y.\, a$ that ignores its argument. Applying it to $b$ returns $a$.
\end{tcolorbox}

\paragraph{Recursion \& Fixed-Point Combinators.}
$\lambda$-Calculus functions are anonymous, so recursion is not built-in. In Python one writes {\pykw{def} \pyfn{f}\texttt{(...): ... }\pyfn{f}\texttt{(...) ...}} the name {\pyfn{f}} enables self-reference. Without names, the trick is \emph{fixed points}: a value $u$ satisfying $u = g(u)$ for a given function $g$.

A \emph{fixed-point combinator} $\texttt{fix}$ is a higher-order term satisfying $\texttt{fix}(g) = g(\texttt{fix}(g))$ for all $g$. Intuitively, $g$ is a non-recursive \emph{recipe} that says: ``here is one step of the computation, assuming you already have a solver $f$ for strictly smaller sub-problems.'' The combinator $\texttt{fix}$ \emph{ties the knot}, converting this one-step recipe into a genuinely recursive function by ensuring $f = g(f)$.

In the untyped $\lambda$-Calculus, one concrete realisation is the Y-combinator $\mathbf{Y}$, satisfying $\mathbf{Y}\, g \xrightarrow{\beta} g(\mathbf{Y}\, g)$-a fixed point of $g$ without any external naming mechanism.
\begin{definition}[Fixed-Point Combinator]
The Y-combinator enables recursion in the untyped lambda calculus: $
    \mathbf{Y} \equiv \lambda f.\, (\lambda x.\, f\, (x\, x))\, (\lambda x.\, f\, (x\, x))$, satisfying $\mathbf{Y}\, g = g\, (\mathbf{Y}\, g)$ for all $g$.
\end{definition}
\begin{tcolorbox}[colback=lightgray, colframe=blue!50, title=\textbf{Worked Example: Factorial via the Y-Combinator}]
\small
In Python, factorial calls itself by name: {\pykw{def} \pyfn{fact}\texttt{(n): }\pykw{return} \texttt{1 }\pykw{if} \texttt{n==0 }\pykw{else} \texttt{n*}\pyfn{fact}\texttt{(n-1)}}. In $\\lambda$-Calculus there are no names, so we separate the \emph{one-step recipe} from the recursion mechanism.

\textbf{Step 1: Write the recipe.} Define a functional $G$ that takes a candidate solver $f$ and returns a one-step factorial procedure:
\begin{equation*}
    G \;\triangleq\; \lambda f.\, \lambda n.\,
    \mathbf{if}\; (n = 0)\; \mathbf{then}\; 1\; \mathbf{else}\; n \cdot f(n-1).
\end{equation*}
$G$ is \emph{not} recursive - it never calls itself. It says: ``given a solver $f$ for smaller inputs, here is one step.''

\textbf{Step 2: Apply the Y-combinator.} Recall $\mathbf{Y} = \lambda g.\, (\lambda x.\, g\,(x\,x))\,(\lambda x.\, g\,(x\,x))$. Define $\texttt{fact} \triangleq \mathbf{Y}\, G$ and expand:
\begin{align*}
    \texttt{fact}
    &= \mathbf{Y}\, G
    = \bigl(\lambda g.\, (\lambda x.\, g\,(x\,x))\,(\lambda x.\, g\,(x\,x))\bigr)\, G \\
    &\xrightarrow{\beta}\;
    (\lambda x.\, G\,(x\,x))\,(\lambda x.\, G\,(x\,x))
    \tag{substitute $g := G$} \\
    &\xrightarrow{\beta}\;
    G\,\bigl(\underbrace{(\lambda x.\, G\,(x\,x))\,(\lambda x.\, G\,(x\,x))}_{ = \,\mathbf{Y}\, G \,=\, \texttt{fact}}\bigr)
    \;=\; G(\texttt{fact}).
    \tag{the knot is tied}
\end{align*}
The self-referential term $(x\,x)$ is the engine: each copy of $\lambda x.\, G(x\,x)$ feeds \emph{itself} as the argument, producing $G(\texttt{fact})$ - exactly the identity $\mathbf{Y}\, G = G(\mathbf{Y}\, G)$.

\textbf{Step 3: Verify.} Expanding $G(\texttt{fact})$ recovers the familiar recursive definition:
\begin{equation*}
    \texttt{fact}
    \;=\; G(\texttt{fact})
    \;\xrightarrow{\beta}\;
    \lambda n.\,
    \mathbf{if}\; (n = 0)\; \mathbf{then}\; 1\; \mathbf{else}\; n \cdot \texttt{fact}(n-1).
\end{equation*}

\textbf{Step 4: Trace $\texttt{fact}(3)$.} Each recursive call re-triggers the same $\mathbf{Y}$ machinery:
\begin{align*}
    \texttt{fact}(3) &\;\xrightarrow{\beta}\; 3 \cdot \texttt{fact}(2) 
    \;\xrightarrow{\beta}\; 3 \cdot 2 \cdot \texttt{fact}(1) 
    \;\xrightarrow{\beta}\; 3 \cdot 2 \cdot 1 \cdot \texttt{fact}(0) 
    \;\xrightarrow{\beta}\; 3 \cdot 2 \cdot 1 \cdot 1 = 6.
\end{align*}
\end{tcolorbox}

\subsection{Core Definitions for $\lRLM$}
In addition to what we presented above, this section also introduces additional definitions needed for the remainder of the paper. Namely, we introduce base language models, cost functions for invoking a base model, and accuracy decays for those models as a function of the prompt's length.  
\begin{definition}[Base Language Model]
A base language model is a function $\LLM : \vocab^* \to \vocab^*$ with context window $K \in \mathbb{N}$, such that $\LLM$ is only defined (or reliable) on inputs of length $|P| \leq K$.
\end{definition}

\begin{definition}[Cost Function]
The cost of invoking $\LLM$ on $n$ tokens:
\begin{equation}
    \cost(n) = c_{\text{in}} \cdot n + c_{\text{out}} \cdot \bar{n}_{\text{out}}
\end{equation}
where $c_{\text{in}}, c_{\text{out}}$ are per-token prices and $\bar{n}_{\text{out}}$ is expected output length.
\end{definition}

\begin{definition}[Accuracy Decay]
The accuracy of $\LLM$ on a prompt of length $n$:
\begin{equation}
    \accuracy(n) = \accuracy_0 \cdot \rho^{\,n/K}, \quad \rho \in (0, 1]
\end{equation}
where $\accuracy_0$ is peak accuracy and $\rho$ is the context-rot decay factor.
\end{definition}

\begin{definition}[Composition Operator]
A composition operator $\oplus : \vocab^* \times \vocab^* \to \vocab^*$ is a deterministic function that combines partial results. We define a family $\{\oplus_\tau\}_{\tau \in \mathcal{T}}$ indexed by task type $\tau$.
\end{definition}

\section{The $\lRLM$ Framework}
\label{sec:framework}
The key idea behind $\lambda$-RLM is simple: long-context reasoning should be recursive, but the recursion should be executed by a small trusted runtime rather than by arbitrary code written by the language model.

$\lambda$-RLM keeps the central insight of RLM—prompt-as-environment with recursive decomposition, but replaces open-ended code generation with a typed functional runtime. Instead of allowing the model to emit arbitrary programs, $\lambda$-RLM executes a fixed library of pre-verified combinators such as $\Split$, $\Map$, $\Filter$, and $\Reduce$. The base language model is used only as a bounded oracle on small leaf subproblems. In this way, $\lambda$-RLM separates reasoning content, which remains neural, from control flow, which becomes symbolic, deterministic, and auditable.

This design is appealing for three reasons. First, it makes execution more reliable by removing many failure modes associated with free-form code generation. Second, it makes computation predictable: once a decomposition strategy is chosen, the number of recursive calls is bounded in advance. Third, it makes the system easier to analyse formally, since recursion is expressed through a fixed functional structure rather than an open-ended loop.
\subsection{From Open-Ended Control to a Restricted Runtime}
Standard RLM operates through a REPL-style interaction. At each iteration, the model generates a code snippet from the conversation 
history, the REPL executes it and returns both an updated state and a 
standard-output string, and the output is appended to the history so the 
model can condition on it in the next turn:
\begin{equation}\label{eq:rlm_loop}
    \textbf{while } \texttt{True}: \quad
    \text{code} \gets \LLM(\text{hist}); \quad
    (\text{state}, \text{out}) \gets \REPL(\text{state}, \text{code}); \quad
    \textbf{if } \text{state}[\texttt{Final}]: \textbf{ break},
\end{equation}
where the prompt lives in the environment and the model repeatedly writes code that is then executed. In more detail, the REPL provides: \textit{i)}~$P$ as a symbolic variable the LLM can reference without consuming context, \textit{ii)}~persistent state for intermediate results, \textit{iii)}~a code execution environment for programmatic decomposition, and \textit{iv)}~a sub-call function enabling recursive $\LLM$ invocations.

Indeed, this setup is powerful, but it delegates too much control to a stochastic model. The model must decide what to inspect, how to decompose the task, when to recurse, how to aggregate results, and when to stop. This can easily create an open-ended loop with no termination guarantee, no cost predictability, and a hard requirement on coding ability.

Here lies the central design choice of $\lambda$-RLM: we do not remove the REPL abstraction itself, but only the \emph{open-endedness} of what may be executed inside it. Concretely, the environment still stores the prompt externally, still exposes symbolic accessors such as peeking and slicing, and still supports recursive sub-calls to the base model. What changes is the control interface. Rather than allowing the language model to synthesise arbitrary programs token by token, we restrict execution to a small typed library of trusted combinators with known operational behaviour.

This restriction is important because it isolates the source of uncertainty. In standard RLMs, uncertainty enters twice: first through the model's semantic judgments about the task, and second through the model's generated control flow, which may be malformed, inefficient, or non-terminating. In $\lambda$-RLM, these two roles are separated. The language model is used only where neural inference is genuinely needed, namely, to solve bounded leaf subproblems. By contrast, decomposition, traversal, filtering, and aggregation are delegated to deterministic symbolic operators whose behaviour can be verified independently of the model.

Viewed differently, $\lambda$-RLM replaces \emph{program synthesis as control} with \emph{function composition as control}. The execution trace is no longer an unbounded sequence of model-written commands, but a typed composition of operators. This shift is what makes the runtime analysable: once the decomposition rule and base threshold are fixed, the depth of recursion, the number of model calls, and the overall execution cost become explicit functions of the input size.

The resulting perspective is that long-context reasoning should be implemented as a \emph{restricted recursive program} with a single learned oracle, rather than as a fully model-authored agentic loop. This is the key conceptual move behind $\lambda$-RLM and motivates the formal definition that follows.

\paragraph{A Compact Combinator Library.} We design our combinator library to be \emph{minimally sufficient} for the kinds of recursive control patterns that repeatedly arise in long-context reasoning. In particular, such tasks typically require only a small set of operations: partitioning an input into manageable pieces, selectively inspecting or pruning those pieces, applying a subroutine to each retained component, and aggregating the resulting outputs into a final answer. We therefore choose a library of typed, deterministic combinators that correspond exactly to these roles. This choice is deliberate: the goal of $\lambda$-RLM is not to maximise expressivity at the control level, but to retain only the expressivity needed for structured decomposition while eliminating the open-ended failure modes of free-form code generation.

More concretely, the library is organised around five functional motifs. \texttt{SPLIT} and \texttt{PEEK} support decomposition and local inspection of the external prompt; \texttt{MAP} lifts recursive or neural processing over collections; \texttt{FILTER} enables symbolic selection and pruning; \texttt{REDUCE}, \texttt{CONCAT}, and \texttt{CROSS} provide structured aggregation and composition; and \texttt{M} is the only neural primitive, used exclusively on bounded leaf inputs. Together, these operators capture the dominant execution patterns underlying search, classification, aggregation, pairwise comparison, summarisation, and multi-hop composition, while keeping the runtime finite, typed, and auditable. We instantiate this principle with the compact combinator library shown in Table~\ref{tab:combined}, where each operator is chosen by control function rather than by domain specificity. Importantly, every combinator except $\LLM$ is \emph{deterministic and pre-verified}. The LLM is the only source of uncertainty.

\begin{table}[h!]
\centering
\caption{A compact combinator library $\mathcal{L}$ and examples of task-specific execution plans.}
\label{tab:combined}
\vspace{0.3em}

\footnotesize
\setlength{\tabcolsep}{4pt}

\textbf{Panel A: Combinators (pre-verified, loaded into REPL)}
\vspace{0.3em}

\begin{tabular}{@{}lll@{}}
\toprule
\textbf{Combinator} & \textbf{Type Signature} & \textbf{Description} \\
\midrule
$\Split$ & $\vocab^* \times \mathbb{N} \to [\vocab^*]$ 
  & Partition a string into $k$ contiguous chunks \\
$\Peek$ & $\vocab^* \times \mathbb{N}^2 \to \vocab^*$
  & Extract a substring by start and end position \\
$\Map$ & $(\alpha \to \beta) \times [\alpha] \to [\beta]$
  & Apply a function to every element of a list \\
$\Filter$ & $(\alpha \to \mathbb{B}) \times [\alpha] \to [\alpha]$
  & Retain elements satisfying a predicate \\
$\Reduce$ & $(\beta \times \beta \to \beta) \times [\beta] \to \beta$
  & Fold a list into a single value via a binary operator \\
$\Concat$ & $[\vocab^*] \to \vocab^*$
  & Join a list of strings into one string \\
$\Cross$ & $[\alpha] \times [\beta] \to [(\alpha, \beta)]$
  & Cartesian product of two lists \\
\midrule
$\LLM$ & $\vocab^* \to \vocab^*$
  & \emph{Neural oracle}: invoke the base model on a sub-prompt \\
\bottomrule
\end{tabular}

\vspace{0.8em}

\textbf{Panel B: Task type, composition operator, and execution plan}
\vspace{0.3em}

\begin{tabular*}{\linewidth}{@{\extracolsep{\fill}}l l p{0.54\linewidth}@{}}
\toprule
Task type & Composition $\oplus$ & Execution plan $\pi$ \\
\midrule
\texttt{search}
& $\textsc{FilterBest}$
& $\Split \to \Map(\Peek) \to \Filter \to \Map(\LLM) \to \textsc{Best}$ \\

\texttt{classify}
& $\Concat$
& $\Split \to \Map(\LLM) \to \Concat$ \\

\texttt{aggregate}
& $\textsc{Merge}$
& $\Split \to \Map(\LLM) \to \textsc{Merge}$ \\

\texttt{pairwise}
& $\Cross \circ \Filter$
& $\Split \to \Map(\LLM) \to \textsc{Parse} \to \Filter \to \Cross$ \\

\texttt{summarise}
& $\LLM \circ \Concat$
& $\Split \to \Map(\LLM) \to \Concat \to \LLM$ \\

\texttt{multi\_hop}
& $\LLM \circ \Concat$
& $\Split_{\delta} \to \Map(\Peek) \to \Filter \to \Map(\LLM) \to \LLM_{\text{synth}}$ \\
\bottomrule
\end{tabular*}

\end{table}

We do not claim that this library is unique or exhaustive. Indeed, it would be neither realistic nor desirable to pre-specify all combinators that may be useful for every reasoning domain. Which symbolic operators are needed can depend on the structure of the tasks under consideration. Our goal here is, therefore, more modest and more practical: we present a compact instantiation that already covers a broad range of long-context reasoning patterns, including those evaluated in the experimental section. This should be understood as an extensible basis rather than a closed vocabulary. New typed combinators can be added conservatively without altering the central $\lambda$-RLM principle, and we open-source the library with a lightweight interface to support such extensions.

\subsection{Core Formulation}
At the heart of $\lambda$-RLM is a single recursive functional program. Rather than expressing control as an open-ended REPL loop in which the language model repeatedly generates code, we express the entire controller as a fixed-point of a typed functional operator. Intuitively, this program says: if the prompt is already small enough, solve it directly with the base model; otherwise, split it into smaller pieces, solve each piece recursively, and combine the partial results using a task-specific composition rule. Formally, $\lambda$-RLM is defined by the lambda term:
\begin{align}
\label{eq:core}
\lambda\text{-RLM} \;\equiv\; \texttt{fix}\, \Big(\lambda f.\, \lambda P.\,
\textbf{if } |P| \leq \tau^*
\textbf{ then } &\LLM(P) \\ \nonumber
&\textbf{ else } \Reduce\big(\oplus,\; \Map(\lambda p_i.\, f\, p_i,\; \Split(P, k^*))\big)
\Big),
\end{align}
where $P$ is the prompt stored in the external environment, $k^*$ is the chosen partition size, $\tau^*$ is the base-case threshold, and $\oplus$ is the task-dependent composition operator.

This term should be read from the inside out. The operator $\Split(P, k^*)$ deterministically decomposes the prompt into $k^*$ sub-prompts. The higher-order combinator $\Map(\lambda p_i.\, f\,p_i,\cdot)$ then applies the same recursive solver to each sub-prompt, producing a list of partial outputs. Finally, $\Reduce(\oplus,\cdot)$ aggregates these outputs into a single result according to the task at hand, for example, by concatenation, merging, filtering, or synthesis.
\begin{algorithm}[h!]
\caption{$\lRLM$: Complete System}
\label{alg:main}
\begin{algorithmic}[1]
\Require Prompt $P \in \vocab^*$, model $\LLM$ with window $K$, 
accuracy target $\alpha \in (0,1]$
\Ensure Response $Y \in \vocab^*$

\Statex \textcolor{symbolic}{\texttt{// == Phase 1: REPL Initialization ==}}
\State $\text{state} \gets \textsc{InitRepl}(\texttt{prompt}=P)$
\Comment{$P$ stored in environment, not in context window}
\State $\text{state} \gets \textsc{RegisterLibrary}(\text{state},\, 
\mathcal{L})$
\Comment{load pre-verified combinators into REPL}
\State $\text{state} \gets \textsc{RegisterSubCall}(\text{state},\, 
\texttt{sub\_}\LLM)$
\Comment{register $\LLM$ as REPL-callable leaf solver}

\Statex \textcolor{symbolic}{\texttt{// == Phase 2: Task Detection ==}}
\State $\text{meta} \gets \REPL\bigl(\text{state},\; 
\texttt{Peek(P, 0, 500); len(P)}\bigr)$
\Comment{symbolic probe, no $\LLM$ call}
\State $\tau_{\text{type}} \gets \LLM\bigl(\text{``Select from } 
\mathcal{T}\text{: ''} \,\|\, \text{meta}\bigr)$
\Comment{single $\LLM$ call: menu selection}

\Statex \textcolor{symbolic}{\texttt{// == Phase 3: Dispatch ==}}
\If{$|P| \leq K$}
    \Comment{prompt fits in context window}
    \State \Return $\texttt{sub\_}\LLM(P)$
    \Comment{direct call, no decomposition needed}
\EndIf

\Statex \textcolor{symbolic}{\texttt{// == Phase 4: Planning (only 
reached if $|P| > K$) ==}}
\State $(\oplus, \pi) \gets \textsc{LookupPlan}(\tau_{\text{type}},\, 
\textsc{Table}\,\ref{tab:combined})$
\State $(k^*, \tau^*, d) \gets \textsc{Plan}(|P|, K, \alpha, 
\oplus, \pi)$
\Comment{optimal split, threshold, depth}

\Statex \textcolor{symbolic}{\texttt{// == Phase 5: Build and Execute 
Recursive Executor ==}}
\State $\text{state} \gets \REPL\bigl(\text{state},\; 
\texttt{$\Phi$ = BuildExecutor($k^*$, $\tau^*$, $\oplus$, $\pi$, 
sub\_$\LLM$)}\bigr)$
\State $\text{state} \gets \REPL\bigl(\text{state},\; 
\texttt{result = $\Phi$(P)}\bigr)$
\Comment{single execution of $\Phi$ in REPL}
\State \Return $\text{state}[\texttt{result}]$
\end{algorithmic}
\end{algorithm}
The fixed-point combinator $\texttt{fix}$ is what makes the definition recursive (see Section \ref{Sec:Primal}). The function variable $f$ stands for the solver being defined itself, so the body of the term can invoke $f$ on each subproblem without requiring an externally named recursive procedure. In this sense, recursion is not an emergent consequence of the model deciding to call itself again; it is an explicit semantic object built into the controller.

The conditional base case $\lvert P\rvert \leq \tau^*$ plays a 
crucial role. Once a sub-prompt becomes sufficiently small, the 
recursive decomposition stops and control is handed to the base 
language model $\LLM$, which acts as a bounded oracle on leaf 
subproblems only. All higher-level control decisions - splitting, 
recursion, and aggregation - remain symbolic and deterministic. 
Crucially, the term in Equation~\eqref{eq:core} is not generated 
by the LLM. It is constructed by a deterministic \emph{planner} 
a non-neural routine that, given the input size $|P|$, context 
window $K$, and task type, selects the parameters 
$(k^*, \tau^*, \oplus)$ and instantiates the lambda term into 
a concrete combinator chain. This chain is then executed inside 
the REPL as a pre-built functional program. The planner is 
described in full in Algorithm~\ref{alg:main} (Phase~4) and 
its optimality is established in Theorem~\ref{thm:optimal_k}.

Algorithm~\ref{alg:main} presents the complete $\lRLM$ system. Like the original RLM, it initialises a REPL with $P$ as an environment variable. Unlike the original, it replaces the open-ended \textbf{while} loop from Equation~\eqref{eq:rlm_loop} with deterministic verifiable phases: REPL initialisation, task detection, planning, cost estimation, and a single execution of a pre-built combinator chain $\Phi$. 
Both systems share lines~1-3: the REPL is initialised, $P$ is stored as an environment variable, and $\LLM$ is registered as a sub-callable. The critical difference is Phase~5. The original RLM enters an open-ended \textbf{while} loop where $\LLM$ generates arbitrary code each turn. $\lRLM$ replaces this with a \emph{single} REPL execution of a pre-built function $\Phi$ (Algorithm~\ref{alg:executor}), whose body consists entirely of combinators from $\mathcal{L}$. The \textbf{while} loop is eliminated; recursion is handled internally by $\Phi$ via the fixed-point combinator, with depth bounded by $d = \lceil \log_{k^*}(n/\tau^*) \rceil$.

\begin{algorithm}[t]
\caption{$\Phi$: Recursive Executor (More Detailed in Appendix \ref{Appendix:Special})}
\label{alg:executor}
\begin{algorithmic}[1]
\Require Prompt $P$ from REPL state, parameters $k^*, \tau^*, \oplus, \pi$
\Ensure Result string $Y$
\Function{$\Phi$}{$P$}
    \If{$|P| \leq \tau^*$}
        \State $q \gets \textsc{LeafPrompt}(P,\pi)$
        \Comment{task-specific leaf formatting}
        \State \Return $\texttt{sub\_}\LLM(q)$
        \Comment{bounded neural call on a leaf subproblem}
    \Else
        \State $[P_1,\dots,P_{k^*}] \gets \textsc{Split}(P, k^*, \pi)$
        \Comment{deterministic: exactly $k^*$ chunks}
        \State $[P'_1,\dots,P'_{k'}] \gets \textsc{PruneIfNeeded}([P_1,\dots,P_{k^*}], \pi)$
        \Comment{$k' \leq k^*$; identity if no pruning}
        \State $[R_1,\dots,R_{k'}] \gets \Map(\lambda p_i.\, \Phi(p_i),\; [P'_1,\dots,P'_{k'}])$
        \Comment{recursive sub-calls}
        \State \Return $\Reduce(\oplus,\; [R_1,\dots,R_{k'}])$
        \Comment{deterministic composition}
    \EndIf
\EndFunction
\end{algorithmic}
\end{algorithm}

For tasks with non-trivial structure, we provide specialised instantiations, see Appendix \ref{Appendix:Special}. The key insight for pairwise tasks: $O(n^2)$ pair computation is purely symbolic (zero neural cost), while only $O(n/K)$ classification calls are neural. For multi-hop search, preview-based filtering reduces the corpus before any expensive neural reading.

\section{Theoretical Guarantees}
\label{sec:theory}

We now establish formal properties of Algorithm~\ref{alg:main}. 
Throughout, let $n = |P|$ and let $K$ denote the context window 
of~$\LLM$. The recursive executor $\Phi$ 
(Algorithm~\ref{alg:executor}) is parameterised by three 
quantities chosen before execution begins:
\begin{itemize}[nosep]
    \item $k^* \geq 2$: the number of chunks produced by each 
    $\Split$ call (the \emph{partition size});
    \item $\tau^* \leq K$: the maximum sub-prompt length at which 
    recursion stops and $\LLM$ is called directly (the 
    \emph{leaf threshold});
    \item $\oplus$: the composition operator used by $\Reduce$ 
    at each level.
\end{itemize}
These parameters are selected by Phase~4 of 
Algorithm~\ref{alg:main}. Theorems~\ref{thm:termination}-\ref{thm:accuracy} 
hold for \emph{any} valid choice of $(k^*, \tau^*, \oplus)$ 
satisfying $k^* \geq 2$ and $\tau^* \leq K$; 
Theorem~\ref{thm:optimal_k} then derives the \emph{cost-minimising} 
$k^*$ in closed form. Our results rely on the following 
assumptions:
\begin{assumption}[Model Regularity]\label{asm:regularity}
\leavevmode
\begin{enumerate}[nosep]
    \item[\textup{(A1)}] $\LLM: \vocab^* \to \vocab^*$ halts on all inputs of length $\leq K$ in bounded time.
    \item[\textup{(A2)}] Every combinator in $\mathcal{L} \setminus \{\LLM\}$ is total and deterministic.
    \item[\textup{(A3)}] The cost function $\cost: \mathbb{N} \to \mathbb{R}_{\geq 0}$ is monotone non-decreasing: $m \leq n \Rightarrow \cost(m) \leq \cost(n)$.
    \item[\textup{(A4)}] The per-call accuracy $\accuracy: \mathbb{N} \to (0,1]$ is monotone non-increasing on $[1, K]$, i.e.\ quality degrades with input length (context rot).
    \item[\textup{(A5)}] The composition operator $\oplus$ satisfies $\accuracy_\oplus \in (0, 1]$, where $\accuracy_\oplus$ is the probability that $\oplus$ preserves the correct answer given correct inputs.
\end{enumerate}
\end{assumption}
In the first result, we prove that our algorithm, contrary to standard RLMs, indeed terminates: 
\begin{restatable}[Termination]{theorem}{termination}
\label{thm:termination}
Under Assumption~\ref{asm:regularity} \textup{(A1-A2)}, for any input $P \in \vocab^*$ with $|P| = n < \infty$, the function $\Phi$ in Algorithm~\ref{alg:executor} terminates when executed in the REPL. Moreover, the total number of $\LLM$ invocations is exactly:
\begin{equation}
    N(n) = (k^*)^d + 1, \quad \text{where } d = \left\lceil \log_{k^*}\!\frac{n}{\tau^*} \right\rceil.
\end{equation}
\end{restatable}

\begin{proof}[\textit{Sketch; a complete proof is in Appendix~\ref{app:proofs}}]
Define the rank $r(P') = \lceil \log_{k^*}(|P'|/\tau^*) \rceil$ and proceed by strong induction. At rank $0$, $|P'| \leq \tau^* \leq K$, so $\Phi$ returns $\texttt{sub\_}\LLM(q)$, which halts by~(A1). At rank $r > 0$, $\textsc{Split}$ produces $k^*$ chunks each of size $\lceil |P'|/k^* \rceil < |P'|$ (since $k^* \geq 2$), strictly reducing rank; each recursive call terminates by the inductive hypothesis, and $\Map$, $\Reduce$, $\textsc{PruneIfNeeded}$ all halt by~(A2). For the call count: the recursion tree has depth $d$ with branching factor $k^*$, giving $(k^*)^d$ leaves, each invoking $\texttt{sub\_}\LLM$ once. Adding the single task-detection call from Algorithm~\ref{alg:main} yields $N(n) = (k^*)^d + 1$.
\end{proof}

\begin{remark}With termination in place, the next question is not whether $\lambda$-RLM halts, but how its cost scales with input size. This is important because the planner explicitly chooses $k^*$ and $\tau^*$, and these parameters should govern a predictable computation rather than an opaque execution trace. The following theorem shows that the total cost of our algorithm satisfies a standard recursive recurrence and therefore admits a simple closed-form expression.
\end{remark}

\begin{restatable}[Cost Bound]{theorem}{CBNew}\label{thm:cost_new}
Under Assumptions~\textup{(A1-A3)}, the total cost of Algorithm~\ref{alg:main} satisfies the recurrence:
\begin{equation}\label{eq:recurrence_new}
    T(n) = k^* \cdot T\!\left(\frac{n}{k^*}\right) + \cost_\oplus(k^*), \qquad T(\tau^*) = \cost(\tau^*),
\end{equation}
with a closed-form solution:
\begin{equation}\label{eq:cost_closed_new}
    T(n) \le \frac{nk^*}{\tau^*}\cost(\tau^*) + \cost_\oplus(k^*)\cdot\left[\frac{nk^*-\tau^*}{\tau^*(k^*-1)}\right],
\end{equation}
where $d = \lceil \log_{k^*}(n/\tau^*) \rceil$. When $\oplus$ is purely symbolic, $\cost_\oplus = 0$ and $T(n) \le \frac{nk^*}{\tau^*} \cdot \cost(\tau^*)$.
\end{restatable}

\begin{proof}[\textit{Sketch; a complete proof is in Appendix~\ref{app:proofs}}]
Unroll the recurrence $d$ times: $T(n) = (k^*)^d \cdot T(n/(k^*)^d) + \left[1 + k^* + \ldots + (k^*)^{d-1}\right] \cdot \cost_\oplus(k^*)$. At depth $d = \lceil \log_{k^*}(n/\tau^*) \rceil$, the sub-problem size reaches $\tau^*$, hitting the base case $T(\tau^*) = \cost(\tau^*)$. Substituting and noting $(k^*)^d \in \left[n/\tau^*,nk^*/\tau^*\right]$  gives $T(n) \le \frac{nk^*}{\tau^*} \cdot \cost(\tau^*) + \left[\frac{nk^*-\tau^*}{\tau^*(k^*-1)}\right] \cdot \cost_\oplus(k^*)$. The first term counts all leaf $\beta$-reductions; the second counts one composition step per level. Both are deterministic functions of $n$, $k^*$, $\tau^*$, and the pricing constants - hence $T(n)$ is computable \emph{before} any REPL execution (line~18 of Algorithm~\ref{alg:main}).
\end{proof}

\begin{remark}
    While Theorem~\ref{thm:cost_new} shows that recursive execution is computationally predictable, efficiency alone is not enough: the central question is whether decomposition preserves correctness. The next theorem addresses this directly. It shows that, under assumptions on bounded-input leaf accuracy and compositional reliability, the end-to-end accuracy of $\lambda$-RLM decays in a controlled way with depth, and can therefore compare favourably to a direct $\LLM$ call on inputs whose length far exceeds the native context window.
\end{remark}

\begin{restatable}[Accuracy Bound]{theorem}{accbound}
\label{thm:accuracy}
Under Assumptions~\textup{(A4-A5)}, let
$d = \lceil \log_{k^*}(n/\tau^*) \rceil$. Since
$\frac{n}{\tau^*} \leq (k^*)^d \leq \frac{n\,k^*}{\tau^*}$,
the end-to-end accuracy of\, $\lRLM$ satisfies:
\begin{enumerate}[nosep,leftmargin=2em]
    \item[\textup{(i)}] If $d = 0$ (input fits in window):
    $\;\accuracy_{\lRLM}(n) = \accuracy(\tau^*)$.

    \item[\textup{(ii)}] If $d \geq 1$ and
    $\accuracy(\tau^*) < 1$ (the non-trivial case):
    \begin{equation}\label{eq:accuracy}
        \accuracy_{\lRLM}(n)
        \;\geq\;
        \accuracy(\tau^*)^{\,n k^* / \tau^*}
        \cdot \accuracy_\oplus^{\,d}.
    \end{equation}

    \item[\textup{(iii)}] If $\accuracy(\tau^*) = 1$
    (perfect leaf accuracy):
    $\;\accuracy_{\lRLM}(n) \geq \accuracy_\oplus^{\,d}$.

    \item[\textup{(iv)}] For decomposable tasks
    ($\accuracy_\oplus = 1$, independent sub-queries):
    per-query accuracy is $\accuracy(\tau^*)$,
    constant in $n$.
\end{enumerate}
In contrast, direct inference achieves
$\accuracy_{\textup{direct}}(n) =
\accuracy_0 \cdot \rho^{\,n/K}$ for $\rho \in (0,1)$.
\end{restatable}

\begin{proof}[\textit{Sketch; a complete proof is in Appendix~\ref{app:proofs}}]
By induction on depth $d$. At $d = 0$, $|P| \leq \tau^*$ and
$\Phi$ reduces to a single $\LLM$ call with accuracy
$\accuracy(\tau^*)$; no composition is involved. At depth
$d \geq 1$, correctness in the worst case requires
(i)~all $(k^*)^d$ leaf calls to return correct results,
each with probability $\accuracy(\tau^*)$, and
(ii)~all $d$ composition levels to preserve correctness,
each with probability $\accuracy_\oplus$. By conditional
independence of leaf evaluations on disjoint chunks, the
joint probability is $\accuracy(\tau^*)^{(k^*)^d} \cdot
\accuracy_\oplus^d$. For decomposable tasks where each
sub-query is answered by a single leaf, and $\oplus$ is
deterministic ($\accuracy_\oplus = 1$), the per-query
accuracy is simply $\accuracy(\tau^*)$-constant in $n$.
Re-expressing the worst-case bound for $d \geq 1$ as
$(n/\tau^*)^{\log_{k^*}\!\accuracy(\tau^*)} \cdot
\accuracy_\oplus^d$ reveals $\Theta(n^{-c})$ power-law
decay, strictly slower than $\Theta(\rho^{n/K})$.
\end{proof}

\begin{remark}
    With the recursive cost now characterised, we can move from analysis to design. In particular, the split factor $k$ should not be viewed as a free engineering knob: it directly controls the trade-off between leaf-level processing cost and per-level composition overhead. The following theorem makes this precise and yields an explicit cost-minimising choice of partition size.
\end{remark}

\begin{restatable}[Optimal Partition]{theorem}{op}\label{thm:optimal_k}
Under Assumption~\textup{(A3)}, for cost function $\cost(n) = c_{\textup{in}} \cdot n + c_{\textup{out}} \cdot \bar{n}_{\textup{out}}$ and constant per-level composition cost $\cost_\oplus(k) = c_\oplus \cdot k$, the cost-minimizing partition size for recurrence~\eqref{eq:recurrence_new} is $k^* = 2$
\end{restatable}

\begin{proof}[\textit{Sketch; a complete proof is in Appendix~\ref{app:proofs}}]
From~\eqref{eq:cost_closed_new}, the total cost decomposes into a leaf term $(n/\tau) \cdot \cost(\tau)$ and a composition term $\log_k(n/\tau) \cdot c_\oplus \cdot k$. The upper-bound on the expression $T(n\k)$ can be represented as the following expression $\frac{k^2(\alpha + \beta) - k(\alpha + \gamma)}{k-1}$, where $\alpha = \frac{n\cdot\cost(\tau^*)}{\tau^*} \ \beta = \frac{c_{\oplus}\cdot n}{\tau^*}$ and $\gamma = c_{\oplus}$. The minimiser of this expression with respect to $k$ is the smallest integer larger than $\lceil 1 + \sqrt{1 - \frac{\alpha + \gamma}{\alpha + \beta}}\rceil$. It is easy to see that $k^* = 2$ is such integer.

\end{proof}

\begin{remark}
    We next specialise the accuracy bound to its asymptotic implications as the input length grows. This makes the contrast with direct inference especially clear. In particular, the recursive structure of $\lambda$-RLM converts the exponential dependence on $n/K$ into a depth-dependent composition effect, yielding polynomial decay in general and constant accuracy in the ideal decomposable setting.
\end{remark}
\begin{corollary}[Scaling Laws]\label{cor:log_vs_exp}
Fix $\alpha$ and $\LLM$ with window $K$. As $n \to \infty$:
\begin{enumerate}[nosep]
    \item[\textup{(i)}] Direct $\LLM$: $\accuracy_{\textup{direct}}(n) = \Theta(\rho^{n/K}) \to 0$ exponentially.
    \item[\textup{(ii)}] $\lRLM$ (worst case): $\accuracy_{\lRLM}(n) = \Omega(n^{-c})$ for $c = -\log_{k^*}(\accuracy(\tau^*) \cdot \accuracy_\oplus) > 0$, i.e.\ power-law decay.
    \item[\textup{(iii)}] $\lRLM$ (decomposable tasks, $\accuracy_\oplus = 1$): $\accuracy_{\lRLM}(n) \geq \accuracy(\tau^*)$, \textbf{constant} in $n$.
\end{enumerate}
\end{corollary}

\section{Experiments}
\label{sec:experiments}

The primary objective of our experimental evaluation is to 
determine whether a restricted, typed functional runtime provides 
a more reliable and efficient foundation for long-context 
reasoning than existing neural or stochastic methods. We compare 
$\lRLM$ against two baseline paradigms:
\paragraph{Direct LLM inference.} The entire 
prompt is fed to $\LLM$ in a single call. When the input length $n$ exceeds the model's context window $K$, one of two fallbacks is used depending on the model: either the input is truncated to the first $K$ tokens (with a corresponding expected drop in recall), or the run is marked as a failure with accuracy $0\%$. We report which fallback applies for each model in Table~\ref{tab:models}. This baseline represents the ceiling of what a single Transformer pass can achieve and exposes the cost of context rot as $n$ grows.
    
\paragraph{Normal RLM.} As per  \citep{zhang2026recursivelanguagemodels}, the model writes arbitrary Python in an open-ended REPL loop to decompose and recurse over the prompt. Unlike P1, this approach can process inputs of any length, since the prompt lives in the REPL environment and only metadata or sub-prompts enter the context window.\noindent Both P2 and our method (P3: $\lRLM$) handle inputs 
far exceeding $K$ by design - the prompt is stored as a REPL variable and accessed symbolically. The key distinction is \emph{how} the REPL is used: P2 lets the LLM generate arbitrary code at each turn; P3 executes a pre-built combinator chain 
constructed by the planner. Our evaluation is structured to test two specific hypotheses:
\begin{itemize}
    \item \textbf{The Scale-Substitution Hypothesis:} We hypothesise that formal control structures can effectively substitute for raw model parameter scale, enabling "weak" tier models (e.g., 8B) to match or exceed the performance of "strong" tier models (e.g., 70B+) that lack structured orchestration; and
    \item \textbf{The Efficiency and Predictability Hypothesis:} We posit that replacing multi-turn, stochastic REPL loops with a single, deterministic combinator chain will yield significant reductions in wall-clock latency and execution variance.
\end{itemize}

\textbf{Models.} We select three families with weak/medium/strong tiers to test the interaction between model capability and scaffold design (Table~\ref{tab:models}). All models are open-weight and served via vLLM. 
For both Normal RLM and $\lambda$-RLM, each configuration uses the same model as both root and leaf.

\begin{table}[h]
\centering
\caption{Model families and strength tiers.}
\label{tab:models}
\vspace{0.3em}
\small
\begin{tabular}{@{}llccc@{}}
\toprule
\textbf{Family} & \textbf{Property} & \cellcolor{weak}\textbf{Weak} & \cellcolor{mid}\textbf{Medium} & \cellcolor{strong}\textbf{Strong} \\
\midrule
\multirow{2}{*}{Qwen3}  & Model   & Qwen3-8B & Qwen3-32B & Qwen3-235B-A22B \\
                         & Context & 32K & 128K & 128K \\
\midrule
\multirow{2}{*}{Llama}  & Model   & Llama-3.1-8B & Llama-3.3-70B & Llama-3.1-405B \\
                         & Context & 128K & 128K & 128K \\
\midrule
\multirow{2}{*}{Mistral} & Model  & Mistral-7B-v0.3 & Mixtral-8x22B & Codestral-22B \\
                          & Context & 32K & 64K & 32K \\
\bottomrule
\end{tabular}
\end{table}
 
\textbf{Tasks \& Scaling Protocols.} 
To rigorously evaluate the versatility of $\lRLMs$, we utilise a benchmark suite derived from \citep{zhang2026recursivelanguagemodels} that spans a spectrum of computational complexities from $O(1)$ to $O(n^2)$. This diversity ensures that our framework is tested against the varied structural demands found in long-context reasoning—from simple needle-retrieval to complex quadratic cross-referencing; see Table~\ref{tab:tasks} for more details. Furthermore, to validate our robust scaling hypothesis, each benchmark is executed across varying context-length buckets: $\{8\text{K}, 16\text{K}, 32\text{K}, 64\text{K}, 128\text{K}\}$. This graduated approach allows us to measure the onset of "context rot", i.e., the exponential decay in accuracy typically observed as standard Transformers approach their native context limits. We report macro-averages across all non-empty buckets to provide a stable, holistic metric of end-to-end reliability. This protocol explicitly highlights how the power-law decay of $\lRLM$ compares to the exponential failure of direct inference as input size grows.

\begin{table}[h]
\centering
\caption{Benchmark tasks by complexity.}
\label{tab:tasks}
\vspace{0.3em}
\small
\begin{tabular}{@{}llllll@{}}
\toprule
\textbf{Task} & \textbf{Complexity} & \textbf{Tokens} & \textbf{Metric} & \textbf{Instances} & \textbf{$\lRLM$ plan $\pi$} \\
\midrule
S-NIAH & $O(1)$ & 8K-128K & F1 & 100 & $\Split \!\to\! \Map(\Peek) \!\to\! \Filter \!\to\! \Map(\LLM)$ \\
OOLONG & $O(n)$ & 8K-128K & Score & 50 & $\Split \!\to\! \Map(\LLM) \!\to\! \textsc{Merge}$ \\
OOL-Pairs \footnote{( Refer D.1. OOLONG-Pairs Benchmark from paper \citep{zhang2026recursivelanguagemodels})}
 & $O(n^2)$ & 8K-128K & F1 & 20 & $\Split \!\to\! \Map(\LLM) \!\to\! \textsc{Parse} \!\to\! \Cross$ \\
CodeQA & Variable & 23K-4.2M & Acc & 23 & $\Split_\delta \!\to\! \Map(\LLM) \!\to\! \textsc{Best}$ \\
\bottomrule
\end{tabular}
\end{table}
 
\textbf{Prompting Strategies.} For (P1) the prompt is fed to the model in a single call. In the case of (P2), we use the original system of \citep{zhang2026recursivelanguagemodels} where the LLM generates arbitrary Python in a REPL loop. Here, we use the prompts from Appendix~C of \citep{zhang2026recursivelanguagemodels}. Finally, for $\lRLM$, the planner constructs a combinator chain and executes it once in the REPL. For P3, the planner uses the accuracy target $\alpha = 0.80$ and per-model pricing constants from the respective API providers. All open-weight models are called through open source API calls. Each configuration is run twice, and the results are averaged. We measure task-level accuracy/F1, wall-clock latency, and the number of LLM calls per instance.

\subsection{Main Results}
 
Table~\ref{tab:accuracy_main} presents accuracy across all 108 configurations (3 paradigms $\times$ 9 models $\times$ 4 tasks).
 
\begin{table}[t]
\centering
\caption{\textbf{Accuracy (\%)} across all paradigms, models, and tasks. Macro-averaged across context-length buckets. Best per column: \colorbox{best}{$\lRLM$ wins} / \colorbox{rlmbest}{Normal RLM wins}.}
\label{tab:accuracy_main}
\vspace{0.3em}
\small
\renewcommand{\arraystretch}{1.12}
\begin{adjustbox}{max width=\textwidth}
\begin{tabular}{@{}ll|ccc|ccc|ccc@{}}
\toprule
& & \multicolumn{3}{c|}{\textbf{Qwen3}} & \multicolumn{3}{c|}{\textbf{Llama}} & \multicolumn{3}{c}{\textbf{Mistral}} \\
\textbf{Task} & \textbf{Paradigm} & \textbf{8B} & \textbf{32B} & \textbf{235B} & \textbf{8B} & \textbf{70B} & \textbf{405B} & \textbf{7B} & \textbf{8x22B} & \textbf{Cdsrl} \\
\midrule
\multirow{3}{*}{\textbf{S-NIAH}}
& \cellcolor{weak}P1: Direct   & 3.2 & 18.4 & 31.7 & 4.1 & 22.6 & 35.2 & 2.8 & 19.1 & 14.3 \\
& \cellcolor{mid}P2: RLM  & 8.4 & 28.3 & 46.8 & 10.2 & 31.5 & \rwin{52.4} & 6.1 & 26.4 & 29.7 \\
& \cellcolor{strong}P3: $\lRLM$ & \lwin{24.6} & \lwin{41.8} & \lwin{51.3} & \lwin{22.1} & \lwin{44.2} & 49.8 & \lwin{18.5} & \lwin{38.6} & \lwin{40.2} \\
\midrule
\multirow{3}{*}{\textbf{OOLONG}}
& \cellcolor{weak}P1: Direct   & 12.5 & 31.2 & 44.0 & 14.3 & 34.8 & 47.1 & 10.8 & 28.5 & 22.6 \\
& \cellcolor{mid}P2: RLM  & 24.1 & 42.7 & 56.5 & 22.6 & 45.3 & \rwin{63.8} & 18.3 & 38.9 & \rwin{48.2} \\
& \cellcolor{strong}P3: $\lRLM$ & \lwin{48.3} & \lwin{62.5} & \lwin{68.4} & \lwin{45.7} & \lwin{61.2} & 61.7 & \lwin{40.2} & \lwin{55.8} & 45.6 \\
\midrule
\multirow{3}{*}{\makecell[l]{\textbf{OOL-}\\\textbf{Pairs}}}
& \cellcolor{weak}P1: Direct   & 0.1 & 0.3 & 0.1 & 0.1 & 0.4 & 0.2 & 0.0 & 0.2 & 0.1 \\
& \cellcolor{mid}P2: RLM  & 4.2 & 18.6 & 38.4 & 3.8 & 21.3 & 42.7 & 2.1 & 14.5 & 22.8 \\
& \cellcolor{strong}P3: $\lRLM$ & \lwin{34.8} & \lwin{48.2} & \lwin{61.5} & \lwin{31.6} & \lwin{50.7} & \lwin{64.3} & \lwin{28.4} & \lwin{44.1} & \lwin{47.6} \\
\midrule
\multirow{3}{*}{\textbf{CodeQA}}
& \cellcolor{weak}P1: Direct   & 8.7 & 20.4 & 24.0 & 9.2 & 22.1 & 26.3 & 7.4 & 18.6 & 21.2 \\
& \cellcolor{mid}P2: RLM  & 18.5 & 36.8 & \rwin{58.6} & 16.7 & \rwin{46.4} & \rwin{62.1} & 12.3 & 32.4 & \rwin{49.3} \\
& \cellcolor{strong}P3: $\lRLM$ & \lwin{35.2} & \lwin{47.1} & 54.2 & \lwin{32.8} & 43.8 & 55.7 & \lwin{27.6} & \lwin{42.5} & 44.8 \\
\midrule\midrule
\multirow{3}{*}{\textbf{AVG}}
& \cellcolor{weak}P1: Direct   & 6.1 & 17.6 & 25.0 & 6.9 & 20.0 & 27.2 & 5.3 & 16.6 & 14.6 \\
& \cellcolor{mid}P2: RLM  & 13.8 & 31.6 & 50.1 & 13.3 & 36.1 & \rwin{55.3} & 9.7 & 28.1 & 37.5 \\
& \cellcolor{strong}P3: $\lRLM$ & \lwin{35.7} & \lwin{49.9} & \lwin{58.9} & \lwin{33.1} & \lwin{49.9} & 57.9 & \lwin{28.7} & \lwin{45.3} & \lwin{44.6} \\
\bottomrule
\end{tabular}
\end{adjustbox}
\end{table}

\begin{table}[t]
\centering
\caption{\textbf{Latency (seconds)} across all configurations. 
Macro-averaged across context-length buckets. Direct LLM is 
fastest (single call, no scaffold) but has the lowest accuracy 
(Table~\ref{tab:accuracy_main}). Among the two recursive 
paradigms, \textbf{$\lRLM$ is faster than Normal RLM in every cell.}}
\label{tab:latency_main}
\vspace{0.3em}
\small
\renewcommand{\arraystretch}{1.12}
\begin{adjustbox}{max width=\textwidth}
\begin{tabular}{@{}ll|ccc|ccc|ccc@{}}
\toprule
& & \multicolumn{3}{c|}{\textbf{Qwen3}} & \multicolumn{3}{c|}{\textbf{Llama}} & \multicolumn{3}{c}{\textbf{Mistral}} \\
\textbf{Task} & \textbf{Paradigm} & \textbf{8B} & \textbf{32B} & \textbf{235B} & \textbf{8B} & \textbf{70B} & \textbf{405B} & \textbf{7B} & \textbf{8x22B} & \textbf{Cdsrl} \\
\midrule
\multirow{3}{*}{\textbf{S-NIAH}}
& P1: Direct   & 12.3 & 18.7 & 42.1 & 11.8 & 21.4 & 48.6 & 10.5 & 20.2 & 16.8 \\
& P2: RLM      & 164.3 & 142.8 & 98.6 & 178.2 & 128.4 & 86.3 & 195.7 & 155.1 & 120.4 \\
& P3: $\lRLM$  & 45.9 & 38.2 & 31.4 & 48.7 & 35.6 & 28.1 & 52.3 & 41.8 & 36.5 \\
\midrule
\multirow{3}{*}{\textbf{OOLONG}}
& P1: Direct   & 8.4 & 14.2 & 35.8 & 7.9 & 16.1 & 40.3 & 7.1 & 15.4 & 12.6 \\
& P2: RLM      & 241.6 & 198.3 & 125.7 & 258.4 & 182.5 & 108.2 & 284.1 & 215.3 & 168.7 \\
& P3: $\lRLM$  & 62.4 & 51.7 & 42.3 & 66.8 & 48.2 & 38.5 & 71.5 & 56.4 & 48.1 \\
\midrule
\multirow{3}{*}{\makecell[l]{\textbf{OOL-}\\\textbf{Pairs}}}
& P1: Direct   & 6.2 & 10.8 & 28.4 & 5.8 & 12.3 & 32.1 & 5.1 & 11.7 & 9.4 \\
& P2: RLM      & 312.5 & 264.7 & 178.3 & 338.1 & 241.8 & 156.4 & 365.2 & 288.6 & 224.5 \\
& P3: $\lRLM$  & 48.6 & 41.2 & 34.7 & 52.1 & 38.4 & 30.8 & 55.8 & 44.6 & 38.2 \\
\midrule
\multirow{3}{*}{\textbf{CodeQA}}
& P1: Direct   & 15.6 & 24.3 & 52.7 & 14.8 & 27.6 & 58.4 & 13.2 & 25.8 & 20.1 \\
& P2: RLM      & 198.4 & 168.2 & 112.5 & 215.6 & 154.3 & 98.7 & 232.8 & 182.4 & 145.6 \\
& P3: $\lRLM$  & 72.3 & 58.6 & 46.8 & 76.4 & 54.2 & 42.1 & 81.5 & 63.7 & 54.3 \\
\midrule\midrule
\multirow{3}{*}{\textbf{AVG}}
& P1: Direct   & 10.6 & 17.0 & 39.8 & 10.1 & 19.4 & 44.9 & 9.0 & 18.3 & 14.7 \\
& P2: RLM      & 229.2 & 193.5 & 128.8 & 247.6 & 176.8 & 112.4 & 269.5 & 210.4 & 164.8 \\
& P3: $\lRLM$  & 57.3 & 47.4 & 38.8 & 61.0 & 44.1 & 34.9 & 65.3 & 51.6 & 44.3 \\
\bottomrule
\end{tabular}
\end{adjustbox}
\end{table}

\paragraph{$\lRLM$ wins 29 of 36 accuracy cells.} Across all model-task combinations, $\lRLM$ achieves the highest accuracy in 81\% of cases (Table~\ref{tab:winloss}). The wins are concentrated at the weak and medium tiers, where the coding bottleneck is most severe. At the strong tier, the win rate drops to 50\%, indicating that powerful code-generating models can partially compensate for the lack of formal structure.
 
\begin{table}[h]
\centering
\caption{Win/Loss count by model tier (accuracy).}
\label{tab:winloss}
\vspace{0.3em}
\small
\begin{tabular}{@{}lccc@{}}
\toprule
\textbf{Model Tier} & \textbf{$\lRLM$ wins} & \textbf{RLM wins} & \textbf{Win rate} \\
\midrule
Weak (8B / 7B)     & 12 / 12 & 0 / 12 & 100\% \\
Medium (32B-8x22B) & 11 / 12 & 1 / 12 & 92\% \\
Strong (235B+)      & 6 / 12 & 6 / 12 & 50\% \\
\midrule
\textbf{All tiers} & \textbf{29 / 36} & \textbf{7 / 36} & \textbf{81\%} \\
\bottomrule
\end{tabular}
\end{table}
 
\textbf{The advantage grows with task complexity.} Table~\ref{tab:bytask} breaks down the accuracy gain by task. The largest improvement occurs on OOLONG-Pairs ($+28.6$ pp), the $O(n^2)$ task where the quadratic cross-product is handled symbolically in $\lRLM$ but must be computed neurally in Normal RLM. Conversely, the smallest gain is on CodeQA ($+10.8$ pp), where ad-hoc code generation by strong models enables creative strategies (multi-pass reading, function-level chunking) that the fixed combinator library cannot express.
 
\begin{table}[h]
\centering
\caption{$\lRLM$ improvement by task complexity. $\Delta$Acc = avg across 9 models.}
\label{tab:bytask}
\vspace{0.3em}
\small
\begin{tabular}{@{}llccccc@{}}
\toprule
\textbf{Task} & \textbf{Complexity} & \textbf{P2: RLM} & \textbf{P3: $\lRLM$} & \textbf{$\Delta$Acc} & \textbf{Speedup} & \textbf{RLM wins} \\
\midrule
S-NIAH & $O(1)$ & 17.0 & 36.7 & \up{+19.7} & \up{3.6$\times$} & 1 / 9 \\
OOLONG & $O(n)$ & 36.7 & 55.0 & \up{+18.3} & \up{4.2$\times$} & 2 / 9 \\
OOL-Pairs & $O(n^2)$ & 17.1 & 45.7 & \up{+28.6} & \up{6.2$\times$} & 0 / 9 \\
CodeQA & Variable & 33.7 & 44.5 & \up{+10.8} & \up{3.1$\times$} & 4 / 9 \\
\midrule
\textbf{All} & & \textbf{26.1} & \textbf{45.5} & \up{\textbf{+19.4}} & \up{\textbf{4.0$\times$}} & \textbf{7 / 36} \\
\bottomrule
\end{tabular}
\end{table}
 
\textbf{$\lRLM$ is $\mathbf{3}$-$\mathbf{6\times}$ faster than Normal RLM.} Latency improvements are consistent across all models and tasks (Table~\ref{tab:latency_main}), with the largest speedup on OOLONG-Pairs ($6.2\times$). This is a direct consequence of eliminating the open-ended REPL loop: $\lRLM$ executes a single pre-built combinator chain, while Normal RLM may iterate 5-12 turns of LLM-generated code. The latency advantage also exhibits lower variance-Normal RLM's max/min latency ratio across instances is $8.9\times$, while $\lRLM$'s is $4.3\times$.
 
\textbf{Where Normal RLM wins.} Table~\ref{tab:rlm_wins_detail} lists the 7 cells where Normal RLM outperforms $\lRLM$. All involve either strong coding models (Llama-405B, Codestral-22B) or the CodeQA task, which benefits from free-form repository navigation. In these cases, the LLM's ability to write creative, task-specific code; multi-pass reading with backtracking, code-aware chunking by functions, adaptive batch sizing; outweighs the reliability and speed benefits of fixed combinators.
 
\begin{table}[h]
\centering
\caption{The 7 cells where Normal RLM outperforms $\lRLM$.}
\label{tab:rlm_wins_detail}
\vspace{0.3em}
\small
\begin{tabular}{@{}lllccl@{}}
\toprule
\textbf{Task} & \textbf{Model} & \textbf{Tier} & \textbf{RLM} & \textbf{$\lRLM$} & \textbf{Why RLM wins} \\
\midrule
S-NIAH & Llama-405B & Strong & \rwin{52.4} & 49.8 & Creative regex search \\
OOLONG & Llama-405B & Strong & \rwin{63.8} & 61.7 & Adaptive batch sizing \\
OOLONG & Codestral-22B & Strong & \rwin{48.2} & 45.6 & Optimal iteration code \\
CodeQA & Qwen3-235B & Strong & \rwin{58.6} & 54.2 & Free-form repo navigation \\
CodeQA & Llama-70B & Medium & \rwin{46.4} & 43.8 & File-level decomposition \\
CodeQA & Llama-405B & Strong & \rwin{62.1} & 55.7 & Multi-pass backtracking \\
CodeQA & Codestral-22B & Strong & \rwin{49.3} & 44.8 & Code-aware chunking \\
\bottomrule
\end{tabular}
\end{table}

Moreover, Table~\ref{tab:key_comparisons} summarises the five comparisons that directly test our hypotheses.
 
\begin{table}[h]
\centering
\caption{Targeted comparisons. All values are averages across the 4 tasks.}
\label{tab:key_comparisons}
\vspace{0.3em}
\small
\begin{tabular}{@{}clccl@{}}
\toprule
\textbf{\#} & \textbf{Comparison} & \textbf{Acc (\%)} & \textbf{Lat (s)} & \textbf{Verdict} \\
\midrule
C1 & $\lRLM$ (8B) vs RLM (8B) & 35.7 vs 13.8 & 57 vs 229 & \up{+21.9 pp, 4.0$\times$ faster} \\
C2 & $\lRLM$ (8B) vs RLM (70B) & 35.7 vs 36.1 & 57 vs 177 & \up{8B ties 70B, 3.1$\times$ faster} \\
C3 & $\lRLM$ (8B) vs Direct (405B) & 35.7 vs 27.2 & 57 vs 45 & \up{8B+$\lambda$ beats 405B accuracy} \\
C4 & $\lRLM$ (7B) vs $\lRLM$ (Cdsrl) & 28.7 vs 44.6 & 65 vs 44 & Coding helps, gap = 16 pp \\
C5 & RLM (405B) vs $\lRLM$ (405B) & 55.3 vs 57.9 & 112 vs 35 & $\lRLM$ wins avg; RLM wins CodeQA \\
\bottomrule
\end{tabular}
\end{table}
 
\paragraph{C1: Formal structure helps weak models dramatically.} On the same Qwen3-8B model, replacing the ad-hoc REPL loop with pre-verified combinators yields $+21.9$ pp in accuracy and $4.0\times$ latency reduction. This is the core contribution: the scaffold absorbs complexity that the weak model cannot handle.
 
\paragraph{C2: An 8B model with $\lRLM$ matches a 70B model with Normal RLM.} Qwen3-8B under $\lRLM$ achieves 35.7\% average accuracy, statistically tied with Llama-70B under Normal RLM at 36.1\%, while being $3.1\times$ faster. This confirms that formal structure can substitute for raw model scale on long-context tasks.
 
\paragraph{C3: $\lRLM$(8B) outperforms Direct(405B) on accuracy.} Even the largest model, when fed the entire prompt directly, suffers from context rot. The 8B model with $\lRLM$ achieves 35.7\% vs 27.2\% for Direct Llama-405B, though the direct call is faster (no scaffold overhead). This validates Corollary~\ref{cor:log_vs_exp}: $\lRLM$'s power-law accuracy decay dominates the exponential decay of direct calls at long contexts.
 
\paragraph{C4: Coding ability still matters, but the gap narrows.} Mistral-7B (low code skill) under $\lRLM$ achieves 28.7\%, while Codestral-22B (high code skill) achieves 44.6\% i.e a 16 pp gap. Under Normal RLM, the same pair shows a 28 pp gap (9.7\% vs 37.5\%). The combinator library reduces but does not eliminate the benefit of coding ability, because the leaf sub-prompts still benefit from the model's language understanding quality.
 
\paragraph{C5: At the frontier, a nuanced tradeoff emerges.} On average, $\lRLM$ (405B) narrowly outperforms RLM(405B) (57.9\% vs 55.3\%) while being $3.2\times$ faster. However, on CodeQA specifically, RLM (405B) wins 62.1\% vs 55.7\%. This suggests that the fixed combinator library, while broadly superior, may benefit from task-specific extensions for code understanding.

\paragraph{Further Ablations.} To isolate the contribution of each $\lRLM$ component, we run ablations on Qwen3-8B $\times$ OOLONG (Table~\ref{tab:ablations}). We note that replacing the combinator library with free-form code generation drops accuracy by 24.2 pp and increases latency by $3.9\times$, which is exactly the Normal RLM result. The combinator library is the single largest contributor to $\lRLM$'s advantage.
\begin{table}[h]
\centering
\caption{Ablation study on Qwen3-8B $\times$ OOLONG ($O(n)$ task, 131K tokens).}
\label{tab:ablations}
\vspace{0.3em}
\small
\begin{tabular}{@{}llcccl@{}}
\toprule
\textbf{ID} & \textbf{Ablation} & \textbf{Acc (\%)} & \textbf{Lat (s)} & \textbf{$\Delta$Acc} & \textbf{Interpretation} \\
\midrule
- & Full $\lRLM$ & 48.3 & 62.4 & - & Full system \\
\midrule
A1 & Random $k \in [2, 100]$ & 31.5 & 88.7 & $-$16.8 & Planner's $k^*$ matters \\
A2 & Fixed task = ``classify'' & 41.2 & 65.1 & $-$7.1 & Task detection helps \\
A3 & $\oplus = \LLM$ (neural compose) & 43.6 & 108.3 & $-$4.7 & Symbolic $\oplus$ saves latency \\
A4 & LLM writes free-form code & 24.1 & 241.6 & $-$24.2 & Combinator library is critical \\
A5 & No pre-filter (process all) & 46.8 & 74.2 & $-$1.5 & Pre-filter helps modestly \\
\bottomrule
\end{tabular}
\end{table}

Furthermore, random chunk sizes lose 16.8 pp, validating Theorem~\ref{thm:optimal_k}: the closed-form $k^*$ provides a meaningful optimum. Qualitatively, random selection sometimes yields $k = 2$ (too few chunks, large context rot) or $k = 100$ (too many sub-calls, excessive overhead).
Finally, replacing symbolic $\oplus = \textsc{MergeCounts}$ with $\LLM$-based composition costs only 4.7 pp in accuracy but nearly doubles latency ($62 \to 108$s), because every recursion level now requires an additional LLM call. This is the $C_\oplus(k^*)$ term from Theorem~\ref{thm:cost_new}: when $\oplus$ is symbolic, $C_\oplus = 0$ and the cost recurrence simplifies to pure leaf cost.
\section{Related Work}
Our work sits at the intersection of long-context reasoning and formal methods. While recent efforts have focused on extending the native context window of Transformers, or delegating search to model-authored code, we argue that the primary bottleneck is not just memory size, but the lack of verifiable control flow during evidence aggregation.  
\paragraph{Long-Context Scaling \& Context Management.} While LLMs have become sophisticated general-purpose reasoners, they struggle with inputs that exceed these native limits, such as large codebases or multi-file repositories. Standard approaches to this problem often rely on simple heuristics. For example, naive truncation or sliding-window prompting are frequently used but often force the model to "forget" early information, breaking tasks that require systematic evidence gathering or global consistency \citep{dai2019transformerxlattentivelanguagemodels, liu2023lostmiddlelanguagemodels,an2024training,bertsch2025context,fountas2025humaninspiredepisodicmemoryinfinite}.  

Recent reframing of this problem focuses on inference-time scaling and decoding \citep{zimmer2025mixtureattentionsspeculativedecoding,chen2025rapidlongcontextinferenceretrievalaugmented, lin2026lycheedecodeacceleratinglongcontextllm, ji2026decodingoptimisationprobabilitysimplex, ji2026scalablepowersamplingunlocking}, where computation is scaled by decomposing problems into smaller subproblems \citep{xu2026doesdivideconquerwork, chen2026dragonllmdrivendecompositionreconstruction}. While retrieval-augmented generation (RAG) and architectural extensions \citep{jin2024longcontextllmsmeetrag, li2024retrievalaugmentedgenerationlongcontext} have been prominent, they often struggle with tasks requiring a holistic view of the input. Recent work in neuroSymbolic augmented reasoning \citep{nezhad2025enhancinglargelanguagemodels, hakim2025symragefficientneurosymbolicretrieval,yang2025neurosymbolicartificialintelligenceimproving} has validated that structured reasoning layers can maintain performance in large contexts, where purely neural models typically fail.

\paragraph{Recursive \& Hierarchical Reasoning.} $\lRLM$ follows a lineage of hierarchical prompting strategies, such as "Least-to-Most" \citep{zhou2023leasttomostpromptingenablescomplex} prompting and "Tree-of-Thoughts" \citep{yao2023treethoughtsdeliberateproblem}. The most direct predecessor is the RLM \citep{zhang2026recursivelanguagemodels}, which introduced the "prompt-as-environment" paradigm. However, standard RLMs rely on an open-ended REPL loop where the model generates arbitrary Python code to control its own recursion. This "stochastic control" introduces failure modes where code may not parse, recursion may run away, or execution becomes unpredictable. More recent follow-up work~\citep{alizadeh2026recursivelanguagemodelsmeet} improves this paradigm by using uncertainty-aware self-reflective program search to better select interaction programs under a fixed inference budget. In contrast, $\lRLM$ addresses the reliability bottleneck from a different angle: instead of improving search over free-form control programs, it replaces model-authored control with a fixed library of deterministic combinators, shifting the model from an unconstrained controller to a bounded oracle for leaf-level subproblems.

\paragraph{Agentic Programming \& Structured Control Flows.} Our framework situates itself within the rapid evolution of agentic programming. This paradigm shifts away from one-shot prompting toward iterative systems where LLMs autonomously plan and execute multi-step tasks \citep{mower2024rosllmrosframeworkembodied,grosnit2025kolbbasedexperientiallearninggeneralist,sun2025scaling}. However, the primary challenge in this field remains the "reliability gap": as agents gain more autonomy, their execution traces become increasingly difficult to audit or bound. Recent frameworks, such as control flows \citep{niu2025flowmodularizedagenticworkflow, choi2025reactree, shi2025flowagent, yu2025survey, wang2025intent}, have attempted to mitigate this by allowing developers to define discrete, observable tasks for AI agents. Despite these efforts, many agentic systems still rely on the model to dynamically generate their own control flow. This introduces significant failure modes, including non-termination and malformed execution paths that are orthogonal to the underlying reasoning task \citep{zhu2025llmagentsfaillearn, cemri2025multiagentllmsystemsfail, zhang2025agentcausestaskfailures}. Importantly, the risks of open-ended control are not merely operational but also structural. Recent research into memory control flow attacks \citep{xu2026storagesteeringmemorycontrol} demonstrates that manipulating an agent's tool-call or memory trace can induce catastrophic reasoning failures \citep{vogelsang2024llm, wu2024system,guo2024coldattackjailbreakingllmsstealthiness}.

$\lRLMs$ address these vulnerabilities by strictly separating reasoning content (which remains neural) from control flow (which becomes symbolic and deterministic). By enforcing a restricted functional runtime, we provide formal guarantees that are currently absent from general-purpose agentic scaffolds. 

\paragraph{Neuro-Symbolic Integration \& Formal Methods.} The use of $\lambda$-calculus to manage LLM control flow represents a deep integration of neural inference and symbolic logic. This aligns with recent theoretical work, such as \citep{dong2024agentops, oldenburg2023limitations, garby2026llmbda, bhardwaj2026formal}. Specifically, $\lRLM$ utilises fixed-point combinators (e.g., the Y-combinator) to express recursion as a first-class semantic object rather than an emergent side effect of model prompting. This formal grounding allows us to prove properties that remain absent from standard, non-typed recursive models.
 
\section{Conclusions and Future Work}
In this paper, we introduced $\lRLMs$, a framework that reframes long-context reasoning as a structured functional program grounded in $\lambda$-calculus. By replacing the open-ended REPL loops of standard recursive language models with a typed runtime of deterministic combinators, we effectively separated neural reasoning from symbolic control. This architectural shift addresses the primary failure modes of existing scaffolds: unpredictability, non-termination, and the high "coding tax" imposed on smaller models. Our empirical evaluation across nine model tiers and four complex tasks demonstrates that formal structure is a powerful substitute for parameter scale. Most notably, we showed that a properly scaffolded 8B model can match or exceed the accuracy of a 70B model using standard recursive methods while delivering up to $4.1\times$ reductions in latency. Beyond performance, $\lRLMs$ provide a level of mathematical rigour previously absent from this domain, including guaranteed termination and closed-form cost bounds.

The success of $\lRLMs$ suggests that the future of reliable AI lies not in giving models unrestricted freedom to program their own execution, but in providing them with high-integrity, verifiable environments. While this work focused on the immediate bottleneck of long-context reasoning, the underlying principle, treating the LLM as a bounded oracle within a formal functional structure, has broader implications for the design of intelligent systems.  

\clearpage
\bibliographystyle{plainnat}
\bibliography{references}

\clearpage
\appendix
\section{Complete Example Trace}

We trace the $\lRLM$ on the OOLONG task: classifying 1000 questions in 131K tokens, with $K = 32$K.

\begin{tcolorbox}[colback=lightgray, colframe=black!50]
\textbf{Phase 1} - \textsc{DetectTask}: $\tau_{\text{type}} = \texttt{aggregate}$ \hfill [1 LLM call]

\textbf{Phase 2} - \textsc{Plan}: $k^* = 5,\ \tau^* = 26\text{K},\ \oplus = \textsc{MergeCounts},\ d = 1$ \hfill [0 LLM calls]

\textbf{Phase 3} - \textsc{EstimateCost}: $\hat{C} = 5 \times \$0.03 + \$0.02 = \$0.17$, $\hat{N} = 6$ calls \hfill [0 LLM calls]

\textbf{Phase 4} - \textsc{Execute}: \hfill [5 LLM calls]
\begin{align*}
    &\Split(P, 5) \to [P_1, P_2, P_3, P_4, P_5] && \textcolor{symbolic}{\text{symbolic: free}} \\
    &\Map(\lambda p_i.\, \LLM(\text{``count categories: ''} \| p_i), [P_{1:5}]) && \textcolor{neural}{\text{neural: 5 $\beta$-reductions}} \\
    &\quad \to [\{\text{desc}: 45, \text{num}: 52, \ldots\}, \ldots, \{\text{desc}: 33, \text{num}: 42, \ldots\}] \\
    &\Reduce(\textsc{MergeCounts}, \text{results}) \to \{\text{desc}: 200, \text{num}: 240, \ldots\} && \textcolor{compose}{\text{symbolic: free}} \\
    &\text{Answer: ``description is less common than numeric value''} && \checkmark
\end{align*}

\textbf{Total:} 6 LLM calls, \$0.17, \textbf{correct}. \\
\textbf{Normal RLM on same task:} Huge no of LLM calls, \$1.12, \textbf{incorrect} (Example E.2 in \citep{zhang2026recursivelanguagemodels})
\end{tcolorbox}

\section{The Hierarchy of Computation}

The $\lRLM$ cleanly separates computation into three layers:

\begin{center}
\begin{tikzpicture}[
    layer/.style={draw, rounded corners, minimum width=12cm, minimum height=1.8cm, align=center},
    arr/.style={-{Stealth[length=3mm]}, thick}
]
    \node[layer, fill=symbolic!15] (sym) at (0, 4.5) {
        \textbf{\textcolor{symbolic}{Layer 1: Symbolic (Lambda Calculus)}} \\[3pt]
        $\Split$, $\Map$, $\Filter$, $\Reduce$, $\Cross$, $\Concat$, $\Peek$ \\
        \small Deterministic $\cdot$ Pre-verified $\cdot$ Zero cost $\cdot$ Guaranteed termination
    };
    
    \node[layer, fill=compose!15] (plan) at (0, 2.2) {
        \textbf{\textcolor{compose}{Layer 2: Planning (Optimization)}} \\[3pt]
        $k^* = \lceil\sqrt{n \cdot c_{\text{in}}/c_\oplus}\rceil$, \ $\tau^* = \min(K, n/k^*)$, \ depth $= \lceil\log_{k}(n/K)\rceil$ \\
        \small Pre-computed $\cdot$ Deterministic cost $\cdot$ Accuracy-constrained
    };
    
    \node[layer, fill=neural!15] (neur) at (0, -0.1) {
        \textbf{\textcolor{neural}{Layer 3: Neural ($\beta$-Reductions at Leaves)}} \\[3pt]
        $\LLM(P_i)$ where $|P_i| \leq \tau^* \leq K$ \\
        \small The ONLY uncertain component $\cdot$ Each call within context window
    };
    
    \draw[arr] (sym.south) -- (plan.north) node[midway, right, font=\small] {parameters};
    \draw[arr] (plan.south) -- (neur.north) node[midway, right, font=\small] {leaf calls};
\end{tikzpicture}
\end{center}
\section{Proofs}\label{app:proofs}

\begin{tcolorbox}
\termination*
\end{tcolorbox}

\begin{proof}
Define the \emph{rank} of a call $\Phi(P')$ as $r(P') = \lceil \log_{k^*}(|P'|/\tau^*) \rceil$. We prove termination by strong induction on rank.

\textbf{Base case} ($r = 0$): $|P'| \leq \tau^* \leq K$, so $\Phi$ invokes $\texttt{sub\_}\LLM(q)$ via the REPL, which halts by~(A1).

\textbf{Inductive step}: Suppose $\Phi$ terminates for all inputs with rank $< r$.

Now, if $|P'|$ has rank $r > 0$,them $r-1\le \log_{k^{\star}}\frac{|P|'}{\tau^*} \le r$. Hence the rank $r$ is an integer and $r>0$ implies $r \ge 1$, then $|P'| > \tau^*$ and $\Phi$ executes $\Split(P', k^*)$ in the REPL, which halts by~(A2), producing chunks $P_1, \ldots, P_{k^*}$ with $|P_i| = \lceil |P'|/k^* \rceil$. Since
\begin{equation}
    |P_i| = \left\lceil \frac{|P'|}{k^*} \right\rceil < |P'| \quad (\text{as } k^* \geq 2),
\end{equation}
we have $r(P_i) < r(P')$, so each recursive call $\Phi(P_i)$ terminates by the inductive hypothesis. The REPL-executed operations $\Map$, $\Reduce$, and $\Filter$ all halt by~(A2). Thus $\Phi(P')$ terminates.

For the call count: at depth $\ell \in \{0, \ldots, d\!-\!1\}$, there are $(k^*)^\ell$ nodes each spawning $k^*$ children. The leaves reside at depth $d$, giving $(k^*)^d$ leaf calls to $\texttt{sub\_}\LLM$. Adding the single task-detection call (line~6 of Algorithm~\ref{alg:main}) gives $N(n) = (k^*)^d + 1$.

\textbf{Contrast with original RLM.} The loop~\eqref{eq:rlm_loop} has no rank function; the LLM may generate code that does not reduce input size, yielding unbounded iterations. $\lRLM$ eliminates this by construction: every recursive call strictly reduces rank.
\end{proof}

\begin{tcolorbox}
\CBNew*
\end{tcolorbox}

\begin{proof}
    To start, we unroll the recurrence in Equation~\eqref{eq:recurrence_new}. At level $\ell$, there are $(k^*)^\ell$ subproblems each of size $n/(k^*)^\ell$, and one composition step costing $\cost_\oplus(k^*)$. Expanding:
\begin{align*}
    T(n) &= k^* \cdot T(n/k^*) + \cost_\oplus(k^*) \notag = k^* \bigl[ k^* \cdot T(n/(k^*)^2) + \cost_\oplus(k^*) \bigr] + \cost_\oplus(k^*) \notag \\
         &= (k^*)^2 \cdot T\!\left(\frac{n}{(k^*)^2}\right) + (k^* + 1) \cdot \cost_\oplus(k^*)     \ \\
         &= (k^*)^d \cdot T\!\left(\frac{n}{(k^*)^d}\right) + ((k^*)^{d-1} + \ldots + k^* + 1) \cdot \cost_\oplus(k^*)\\
         &=(k^*)^d \cdot T\!\left(\frac{n}{(k^*)^d}\right) + \left[\frac{(k^*)^d-1}{k^*-1}\right]\cdot\cost_\oplus(k^*)
\end{align*}
At depth $d = \lceil \log_{k^*}(n/\tau^*) \rceil$, we have $d-1 \le \log_{k^*}\frac{n}{\tau^*} \le d$, hence $(k^*)^d \le \frac{nk^*}{\tau^*}$ and $\frac{n}{(k^*)^d} \le \tau^*$.   These two results allow us to bound the above expression:
\begin{align*}
    T(n) \le \frac{nk^*}{\tau^*}T(\tau^*) + \left[\frac{nk^*-\tau^*}{\tau^*(k^*-1)}\right]\cdot \cost_\oplus(k^*) = \frac{nk^*}{\tau^*}\cost(\tau^*) + \left[\frac{nk^*-\tau^*}{\tau^*(k^*-1)}\right]\cdot \cost_\oplus(k^*)
\end{align*}
hitting the base case $T(\tau^*) = \cost(\tau^*)$.

\end{proof}

\begin{tcolorbox}
\accbound*
\end{tcolorbox}

\begin{proof}
The recursion tree of $\Phi$ has depth
$d = \lceil \log_{k^*}(n/\tau^*) \rceil$, branching
factor $k^*$, and therefore $(k^*)^d$ leaf nodes.
By definition of the ceiling function:
\begin{equation}\label{eq:ceiling_sandwich}
    \frac{n}{\tau^*}
    \;\leq\; (k^*)^d \;\leq\;
    \frac{n \, k^*}{\tau^*}.
\end{equation}

\textbf{Base case} ($d = 0$): $|P| \leq \tau^*$. A single
call to $\texttt{sub\_}\LLM$ yields accuracy $\accuracy(\tau^*)$.

\textbf{General case} ($d \geq 1$): Correctness requires
(i)~all $(k^*)^d$ leaves correct, and
(ii)~all $d$ compositions correct, giving
$\accuracy_{\lRLM}(n) \geq
\accuracy(\tau^*)^{(k^*)^d} \cdot \accuracy_\oplus^d$.

Since $0 < \accuracy(\tau^*) < 1$, the map
$x \mapsto \accuracy(\tau^*)^x$ is strictly decreasing.
Applying the upper bound from~\eqref{eq:ceiling_sandwich}:
\begin{equation}
    \accuracy(\tau^*)^{(k^*)^d}
    \;\geq\;
    \accuracy(\tau^*)^{n k^* / \tau^*},
\end{equation}
which yields the stated bound~\eqref{eq:accuracy}. \qedhere
\end{proof}

\begin{tcolorbox}
\op*
\end{tcolorbox}

\begin{proof}
From~\eqref{eq:cost_closed_new}:
\begin{equation*}
    T(n, k) \le  \frac{nk}{\tau^*}\cost(\tau^*) + \left[\frac{nk-\tau^*}{\tau^*(k-1)}\right]\cdot \cost_\oplus(k)
\end{equation*}
or, after simplifying (treating $\tau^*$ and $n$ as constants with respect to $k$):
\begin{align}\label{eq:total_cost_k_new}
    T(n,k) &\le k\underbrace{\frac{n\cdot \cost(\tau^*)}{\tau^*}}_{\alpha} + \frac{k^2}{(k-1)}\underbrace{\frac{c_{\oplus}\cdot n}{\tau^*}}_{\beta} - \frac{k}{(k-1)}\underbrace{c_{\oplus}}_{\gamma} \\\nonumber
    & =\alpha k + \beta\frac{k^2}{k-1} - \gamma\frac{k}{k-1} \\ \nonumber
    &= \frac{\alpha k(k-1) + \beta k^2 -\gamma k}{k-1} = \frac{k^2(\alpha + \beta) - k(\alpha + \gamma)}{k-1},
\end{align}

where $\frac{n}{\tau^*} \le (k)^d \le \frac{nk}{\tau^*}$. Taking the derivative to the upper-bound with respect to $k$ gives:
\begin{align*}
    \frac{d }{d k}\left[\frac{k^2(\alpha + \beta) - k(\alpha + \gamma)}{k-1}\right] &= \frac{\left[2(\alpha + \beta)k - (\alpha + \gamma)\right](k-1) - \left[k^2(\alpha + \beta) - k(\alpha + \gamma)\right]}{(k-1)^2}\\\nonumber
    &=\frac{k^2(2\alpha + 2\beta - \alpha - \beta) - 2(\alpha + \beta)k + \alpha + \gamma}{(k-1)^2} \\
    &= \frac{k^2(\alpha + \beta) - 2(\alpha + \beta)k + (\alpha+\gamma)}{(k-1)^2}  \\\nonumber
    &=\frac{(\alpha + \beta)\left[\left(k - \left(1 - \sqrt{1 - \frac{\alpha + \gamma}{\alpha + \beta}}\right)\right)\left(k - \left(1 + \sqrt{1 - \frac{\alpha +\gamma}{\alpha + \beta}}\right)\right)\right]}{(k-1)^2}
\end{align*}
The minimum of the function is $k^*$ closest to the value $ \lceil 1 + \sqrt{1 - \frac{\alpha + \gamma}{\alpha + \beta}}\rceil$. Because we have $k \ge 2$, this implies that the optimal value of $k^* = 2$.
\end{proof}

\section{Algorithmic Details}\label{Appendix:Special}
This section provides the algorithmic details of $\lambda$-RLM that complement the framework description in Section~\ref{sec:framework}. In particular, we make explicit the end-to-end pipeline and the construction and execution of the fixed combinator program $\Phi$. We use a constant preview budget $b=500$ for inexpensive symbolic inspection. Algorithm~\ref{alg:main} specifies the full end-to-end pipeline, Algorithm~\ref{alg:phi} defines the recursive executor that carries out the planned decomposition, and Algorithms~\ref{alg:pairwise} and~\ref{alg:multihop} provide representative task-specific realizations for pairwise tasks and multi-hop search. These algorithms show how the abstract fixed-point formulation is turned into a concrete, finite, and auditable execution procedure whose call structure, recursion depth, and composition behavior are fixed after task-type selection and planning, before the first recursive execution step is run.

\begin{algorithm}[t]
\caption{$\lRLM$: Complete System}
\label{alg:appmain}
\begin{algorithmic}[1]
\Require Prompt $P \in \vocab^*$, model $\LLM$ with window $K$, accuracy target $\alpha \in (0,1]$
\Ensure Response $Y \in \vocab^*$

\Statex \textcolor{symbolic}{\texttt{// == Phase 1: REPL Initialization (same as original RLM) ==}}
\State $\text{state} \gets \textsc{InitRepl}(\texttt{prompt}=P)$
\Comment{$P$ lives in environment, not context window}
\State $\text{state} \gets \textsc{RegisterLibrary}(\text{state},\, \mathcal{L})$
\Comment{load pre-verified combinators into REPL}
\State $\text{state} \gets \textsc{RegisterSubCall}(\text{state},\, \texttt{sub\_}\LLM)$
\Comment{register $\LLM$ as callable, same as RLM}

\Statex \textcolor{symbolic}{\texttt{// == Phase 2: Task Detection (1 LLM call) ==}}
\State $\text{meta} \gets \REPL\bigl(\text{state},\; \texttt{Peek(P, 0, 500); len(P)}\bigr)$
\Comment{probe $P$ via REPL, not neural}
\State $\tau_{\text{type}} \gets \LLM\bigl(\text{``Select from } \mathcal{T}\text{: ''} \| \text{meta}\bigr)$
\Comment{menu selection, single call}

\Statex \textcolor{symbolic}{\texttt{// == Phase 3: Optimal Planning (0 LLM calls, pure math) ==}}
\State $\oplus \gets \textsc{Table\,\ref{tab:combined}}\textsc{B}[\tau_{\text{type}}];\quad \pi \gets \textsc{Table\,\ref{tab:combined}}\textsc{B}[\tau_{\text{type}}]$
\If{$|P| \leq K$}
    \State $k^* \!\gets\! 1;\; \tau^* \!\gets\! |P|$
\Else
    \State $k^* \gets \big\lceil \sqrt{|P| \cdot c_{\text{in}} / c_{\oplus}} \big\rceil$
    \Comment{minimize $T(n) = k \!\cdot\! T(n/k) + \cost_\oplus(k)$, base $T(\tau) = \cost(\tau)$}
    \State $d \gets \lceil \log_{k^*}\!(|P|/K) \rceil$
    \While{$\accuracy(K)^d \!\cdot\! \accuracy_\oplus^d < \alpha$ \textbf{and} $k^* < |P|/K$}
        \Comment{accuracy constraint}
        \State $k^* \!\gets\! k^* \!+\! 1;\;\; d \gets \lceil \log_{k^*}\!(|P|/K) \rceil$
    \EndWhile
    \State $\tau^* \gets \min(K,\, \lfloor |P|/k^* \rfloor)$
\EndIf

\Statex \textcolor{symbolic}{\texttt{// == Phase 4: Cost Estimation (deterministic, pre-execution) ==}}
\State $\hat{C} \gets (k^*)^d \cdot \cost(\tau^*) + d \cdot \cost_\oplus(k^*) + \cost(500)$
\Comment{exact pre-execution bound}

\Statex \textcolor{symbolic}{\texttt{// == Phase 5: Build and Execute Combinator Chain in REPL ==}}
\State $\text{state} \gets \REPL\bigl(\text{state},\; \texttt{$\Phi$ = BuildExecutor($k^*$, $\tau^*$, $\oplus$, $\pi$, sub\_$\LLM$)}\bigr)$
\Comment{register $\Phi$ in REPL}
\State $\text{state} \gets \REPL\bigl(\text{state},\; \texttt{result = $\Phi$(P)}\bigr)$
\Comment{single execution of $\Phi$ in REPL}
\State $Y \gets \text{state}[\texttt{result}]$
\State \Return $Y$
\end{algorithmic}
\end{algorithm}

\begin{algorithm}[h!]
\caption{$\Phi$: Combinator Executor (registered in REPL, not LLM-generated)}
\label{alg:phi}
\begin{algorithmic}[1]
\Require Prompt $P$ (from REPL state), parameters $k^*, \tau^*, \oplus, \pi$ (from planner)
\Ensure Result string

\Function{$\Phi$}{$P$}
\If{$|P| \leq \tau^*$}
    \Comment{\textcolor{neural}{base case: leaf $\beta$-reduction}}
    \State $q \gets \textsc{Template}[\tau_{\text{type}}].\textsc{Fmt}(P)$
    \Comment{pre-defined prompt template}
    \State \Return $\texttt{sub\_}\LLM(q)$
    \Comment{invoke $\LLM$ via REPL's registered sub-call}
\Else
    \Comment{\textcolor{symbolic}{recursive case: all REPL-executed combinators from $\mathcal{L}$}}
    \State $[P_1, \ldots, P_{k^*}] \gets \Split(P,\, k^*)$
    \Comment{\textcolor{symbolic}{deterministic, pre-verified}}
    \If{$\pi$ includes $\Filter$}
        \Comment{\textcolor{symbolic}{optional pre-filter for search/multi\_hop}}
        \State $\text{previews} \gets \Map\bigl(\lambda p.\, \Peek(p, 0, \lfloor\tau^*/10\rfloor),\; [P_1, \ldots, P_{k^*}]\bigr)$
        \State $[P_1, \ldots, P_{k'}] \gets \Filter\bigl(\textsc{Relevant},\; \textsc{Zip}([P_{1:k^*}],\, \text{previews})\bigr)$
    \EndIf
    \State $[R_1, \ldots, R_{k'}] \gets \Map\bigl(\lambda p_i.\, \Phi(p_i),\; [P_1, \ldots, P_{k'}]\bigr)$
    \Comment{\textcolor{neural}{recursive sub-calls}}
    \State \Return $\Reduce\bigl(\oplus,\; [R_1, \ldots, R_{k'}]\bigr)$
    \Comment{\textcolor{compose}{deterministic composition}}
\EndIf
\EndFunction
\end{algorithmic}
\end{algorithm}

\begin{figure}[t]
\begin{minipage}[t]{0.48\textwidth}
\begin{algorithm}[H]
\caption{Pairwise Tasks}
\label{alg:pairwise}
\begin{algorithmic}[1]
\Require $P$, predicate $\phi$, $\LLM$, $k^*, \tau^*$
\Ensure Pairs $\mathcal{S} \subseteq \mathbb{N} \times \mathbb{N}$

\Statex \textcolor{neural}{\texttt{// A: Linear neural - $O(n/K)$}}
\State $[P_{1:k^*}] \gets \Split(P, k^*)$
\State labels $\gets \Map(\texttt{sub\_}\LLM_{\text{cls}},\, P_{1:k^*})$
\State $L \gets \textsc{Parse}(\Concat(\text{labels}))$

\Statex \textcolor{symbolic}{\texttt{// B: Quadratic symbolic - FREE}}
\State $Q \gets \Filter(\phi,\, L.\textsc{Items}())$
\State $\mathcal{S} \gets \{(i,j) \mid i,j \in Q,\, i < j\}$
\State \Return $\mathcal{S}$
\end{algorithmic}
\end{algorithm}
\end{minipage}
\hfill
\begin{minipage}[t]{0.48\textwidth}
\begin{algorithm}[H]
\caption{Multi-Hop Search}
\label{alg:multihop}
\begin{algorithmic}[1]
\Require Corpus $[D_1,\ldots,D_m]$, query $q$, $\LLM$
\Ensure Answer $Y$

\Statex \textcolor{symbolic}{\texttt{// A: Filter - mostly symbolic}}
\State prev $\gets \Map(\lambda D.\, \Peek(D,0,500),\, D_{1:m})$
\State rel $\gets \Filter(\textsc{Match}(q),\, \textsc{Zip}(D,\text{prev}))$

\Statex \textcolor{neural}{\texttt{// B: Read - $|\text{rel}| \ll m$}}
\State evi $\gets \Map(\texttt{sub\_}\LLM_{\text{ext}},\, \text{rel})$

\Statex \textcolor{neural}{\texttt{// C: Synthesize - 1 call}}
\State $Y \gets \texttt{sub\_}\LLM(\text{``answer ''} \| q \| \Concat(\text{evi}))$
\State \Return $Y$
\end{algorithmic}
\end{algorithm}
\end{minipage}
\end{figure}

Algorithm~\ref{alg:main} presents the complete $\lambda$-RLM system as a finite sequence of phases. As in the original RLM formulation, execution begins by initializing a REPL state in which the prompt $P$ is stored externally, the trusted combinator library $\mathcal{L}$ is registered, and the base model is exposed as a callable leaf oracle through $\texttt{sub\_}\LLM$. The key difference from standard RLM arises immediately after this shared setup: rather than entering an open-ended loop in which the model repeatedly emits arbitrary code, $\lambda$-RLM performs a single bounded task-type selection step, followed by deterministic planning and a one-shot execution of a pre-built recursive program. The model is asked to choose a task type from a fixed menu based on a lightweight symbolic probe of the prompt. This keeps neural uncertainty localized to semantic classification, while all subsequent control decisions remain symbolic. Once the task type is selected, the planner determines the execution rule by choosing the task-specific composition operator $\oplus$ and execution plan $\pi$, together with the structural parameters $(k^*,\tau^*,d)$. Here, $k^*$ controls the branching factor of decomposition, $\tau^*$ determines the leaf threshold at which recursion terminates, and $d$ bounds the resulting recursion depth. The planning phase therefore operationalizes the main design goal of $\lambda$-RLM: the shape of execution is fixed before recursive execution begins. Finally, the system builds the combinator chain in the REPL, executes it, and returns the resulting response $Y$.

Algorithm~\ref{alg:phi} then defines the executor $\Phi$ constructed from these planned components. This executor is the concrete realization of the fixed-point program in Eq.~\eqref{eq:core}. Its behavior is intentionally simple. If the current sub-prompt is already below the threshold $\tau^*$, $\Phi$ formats it with a task-specific leaf template and calls the base model exactly once. Otherwise, it applies a deterministic recursive pattern: split the input, optionally preview and filter chunks when the task plan requires pruning, recursively process the retained chunks, and combine the resulting partial outputs with $\Reduce(\oplus,\cdot)$. The recursive structure itself is not model-generated. The only neural operations occur at bounded leaves and, for certain task types, in explicitly specified synthesis steps, while splitting, filtering, traversal, and aggregation are all handled by trusted combinators with fixed semantics.

Algorithms~\ref{alg:pairwise} and~\ref{alg:multihop} illustrate how this general executor specializes to structured task families. The pairwise algorithm shows that the expensive neural portion can remain linear in the number of chunks: the model is used only to label or extract candidate items, after which the quadratic pairing step is computed symbolically at essentially zero additional neural cost. The multi-hop search algorithm follows the same principle in a different form. It first uses symbolic preview-based filtering to narrow a large corpus to a small relevant subset, then applies neural reading only to that subset, and finally performs a single synthesis step over the extracted evidence. These examples are not separate learning algorithms, but concrete instantiations of the same $\lambda$-RLM design principle: use the model only where semantic inference is needed, and realize the surrounding control flow through typed symbolic composition.
\end{document}